\documentclass[preprint,12pt,authoryear]{elsarticle}

\usepackage{amssymb}
\usepackage{amsmath}
\usepackage{amsfonts}       
\usepackage{nicefrac}       
\usepackage{microtype}      
\usepackage{algpseudocode}
\usepackage{lipsum}
\usepackage{fancyhdr}     
\usepackage{verbatim}
\usepackage[ruled]{algorithm2e}
\usepackage{amsthm}

\newtheorem{proposition}{Proposition}[section]

\journal{Neurocomputing}

\begin{document}

\begin{frontmatter}



\title{Diffusion Model Conditioning on Gaussian Mixture Model and Negative Gaussian Mixture Gradient}


\author[lable1]{Weiguo Lu}
\author[lable1]{Xuan Wu}
\author[lable1]{Deng Ding}
\author[lable2,lable3]{Jinqiao Duan}
\author[lable1]{Jirong Zhuang}
\author[lable3,lable4]{Gangnan Yuan\corref{cor1}}
\cortext[cor1]{Corresponding author}
\affiliation[lable1]{organization={University of Macau},
            city={Macau},
            postcode={999078}, 
            country={China}}
\affiliation[lable2]{organization={Great Bay University},
            city={Dongguan},
            postcode={523000}, 
            country={China}}  
\affiliation[lable3]{organization={Great Bay Institute for Advanced Study},
            city={Dongguan},
            postcode={523000}, 
            country={China}}    

\affiliation[lable4]{organization={University of Science and Technology of China},
            city={ Hefei},
            postcode={230026}, 
            country={China}}
            
\begin{abstract}
Diffusion models (DMs) are a type of generative model that has had a significant impact on image synthesis and beyond.They can incorporate a wide variety of conditioning inputs—such as text or bounding boxes—to guide generation.
In this work, we introduce a novel conditioning mechanism that applies Gaussian mixture models (GMMs) for feature conditioning, which helps steer the denoising process in DMs. Drawing on set theory, our comprehensive theoretical analysis reveals that the conditional latent distribution based on features differs markedly from that based on classes. Consequently, feature-based conditioning tends to generate fewer defects than class-based conditioning. We trained two diffusion models with GMM-based conditioning separately. The experimental results support our theoretical findings.
Additionally, we propose a new gradient function named the Negative Gaussian Mixture Gradient (NGMG) and incorporate it into the training of diffusion models alongside an auxiliary classifier. We theoretically demonstrate that NGMG offers comparable advantages to the Wasserstein distance, serving as a more effective cost function when learning distributions supported by low-dimensional manifolds, especially in contrast to many likelihood-based cost functions, such as Kullback-Leibler (KL) divergences.
\end{abstract}

\begin{graphicalabstract}
\begin{figure}[ht]\centering
\includegraphics[scale=0.45]{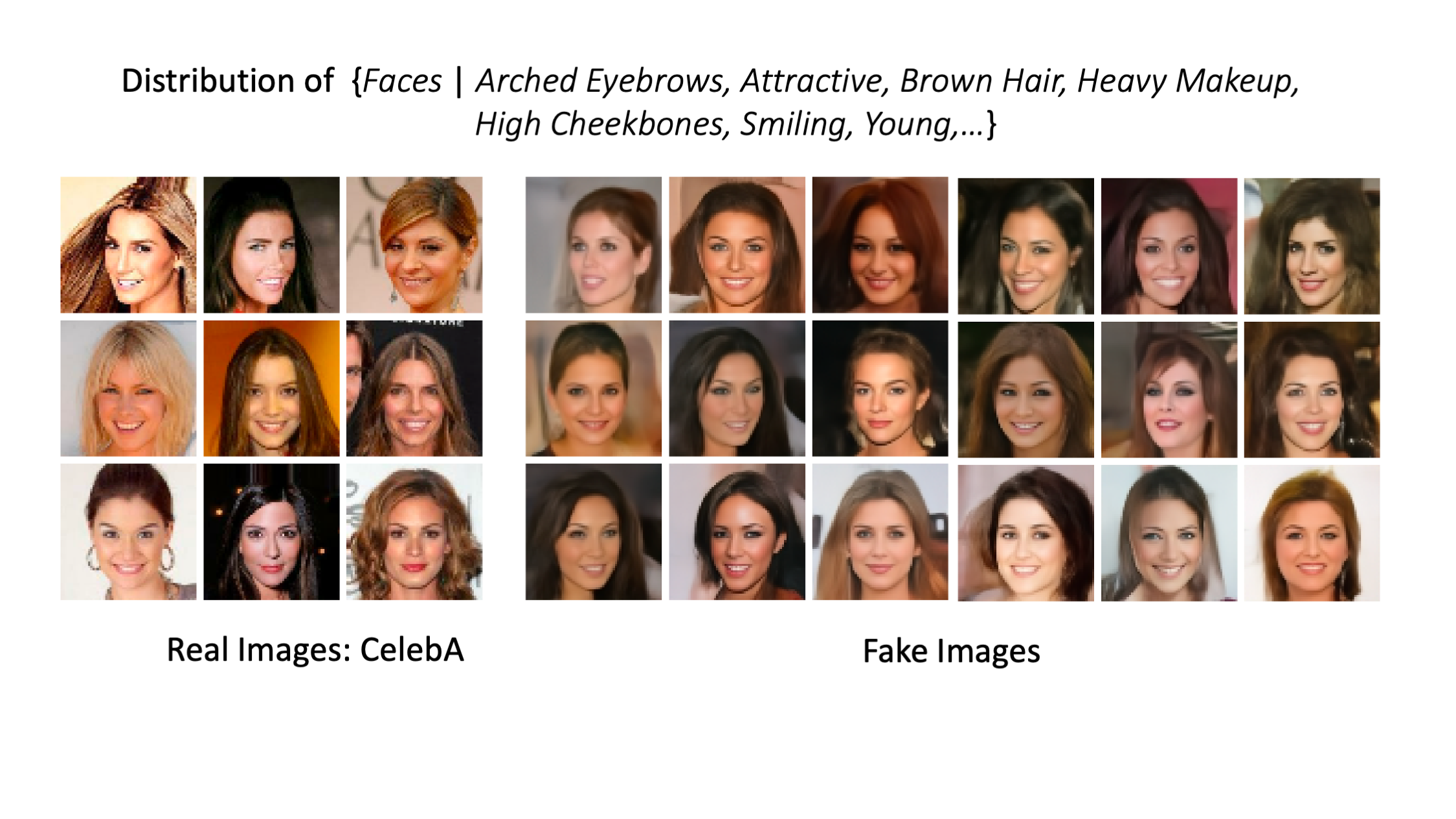}
\caption{Images comparison. Left: real images from CelebA. Right: random generated samples from diffusion model conditioning on Gaussian mixture model. Both of these images have the same feature condition.}
	\label{fig1}
 \end{figure}
\end{graphicalabstract}

\begin{highlights}
\item Research highlight 1: A diffusion model conditioning on Gaussian mixture model is proposed. Latent variables in our model are random variables so that data(real) distribution is approximated by latent distribution.

\item Research highlight 2: Latent distributions can be constructed by classes or features. We show that latent distribution build by features are theoretically better than classes under set theory.

\item Research highlight 3: A classifier is integrated in our diffusion model and a novel distance function is proposed called the negative Gaussian mixture gradient. We theoretically prove that our distance function has the same benefit as the Wasserstein distance which more sensible compared to many likelihood-based cost functions, such as KLs, when learning distributions supported by low-dimensional manifolds.

\item Research highlight 4: Additionally, we binds negative Gaussian mixture gradient together with entropy for Neural network. It can be seen as a combination of Wassersetin distance and likelihood. A direct comparison to binary cross entropy is made, and our method performs significantly better in training.

\end{highlights}


\end{frontmatter}


\section{Introduction}
\label{intro}
Generative models are one of the most spectacular fields in recent neural network development. Diffusion models \citep{Sohl2015Deep, Ho2020Denoising, Song2019Generative, Kingma2021Variational, Song2020Improved, Rombach2022High, Dhariwal2021Diffusion} achieve impressive results in image synthesis, audio generation \citep{Chen2021Wavegard, Kong2020Diffwave, Mittal2021Symbolic}, and super-resolution \citep{Saharia2022Image}. DMs have several benefits compared to GANs, which are commonly known to have mode collapse and training instabilities. The underlying model of DMs is typically a U-Net \citep{Dhariwal2021Diffusion, Ho2020Denoising, Song2020Improved}\citep{Ronneberger2015Unet}, a variation of the autoencoder \citep{kingma2014auto}. While classic autoencoders decode information from latent variables in the bottleneck layer, U-Net integrates information from each encoder layer for decoding. Latent diffusion models \citep{Rombach2022High} train an autoencoder that produces a lower-dimensional latent space to reconstruct the data space \citep{Razavi2019Generating, VanDenOord2017Neural}. Various image-to-image, text-to-image, and token-based conditioning mechanisms can be applied in latent diffusion models. The text-to-image task takes textual information and tokenizes it as conditional input. In \cite{Rombach2022High}'s work, a BERT tokenizer \citep{Devlin2019Bert} is used to infer latent codes. In many conditioning mechanisms, latent variables are commonly used to control generations. 

In this work, we propose a new method that uses random variables for latent. For instance, the concept of `smiling' cannot be reduced to a specific angle of mouth posture. Similarly, when we consider `Beautiful starry night', it represents a concept that is constructed by almost infinite variations under different conditional scenarios. In short, `Beautiful starry night' is a distribution rather than a fixed value. The attributes of the CelebA dataset \citep{Liu2015Face} provide a good example. Features labels such as `Arched Eyebrows', `Brown Hair', `Young', are more akin to concepts than precise descriptors. With this in mind, we propose a variation of the token-based conditioning mechanism using the Gaussian mixture model to construct a distributional conditioning approach. Figure \ref{fig1} illustrates an example of our proposed diffusion model trained on CelebA. From a distributional perspective, given any set of features, a conditional distribution is formed from which images can be randomly sampled. When given conditions such as `Arched Eyebrows', `Attractive', `Brown Hair', etc., our generation should adhere to the corresponding conditional distribution.

Training neural networks with our proposed conditioning mechanisms can be viewed as a method for approximating the event space of the data distribution. We find that latent distributions, constructed based on classes or features, can significantly impact model performance. Unlike class labels, which are typically disjoint, features often overlap. The CelebA dataset does not categorize images into distinct classes but assigns a set of features to each image, effectively dividing the dataset into overlapping subsets based on these features. Latent distributions formed around features can lead to more refined conditional event spaces, yielding generations with fewer defects than those constructed solely around classes. We proved this claim through set theory, and our experiments provide both theoretical and empirical support for our findings.

Incorporating an additional classifier to train a generator has demonstrated substantial enhancements in model performance. The Generative Adversarial Network (GAN) \citep{Goodfellow2020Generative} employs an neural network, known as the discriminator, to adversarially train the generator. Numerous extensions and variations of GANs have achieved state-of-the-art generation quality in a variety of image generation tasks \citep{Wu2019Logan, Karras2020Analyzing, Brock2018Large, Mirza2014Conditional}. Similarly, classification has been leveraged to augment generation quality and to increase the utility of trained latent variables, as evidenced by the literature \citep{Dumoulin2017A, Vries2017Modulating, Miyato2017Cgan, Lucic2019High, Dash2017Tac, Lang2021Explaining, Lu2023Efficient, Dhariwal2021Diffusion}. To capitalize on these advancements, we introduce a modified version of the diffusion model that incorporates a classifier into its training process. Concurrently, we have developed a novel distance function based on the Gaussian mixture model for the effective training of this classifier.

In unsupervised learning, the general goal is to learn a probability mass or density function. Likelihood-based methods, such as KL divergence and Jensen-Shannon divergence, are the two most prevalent techniques. However, these methods have been reported to exhibit instability during the training process. Studies by \cite{Arjovsky2017Towards, Arjovsky2017Wasserstein} indicate that the predicted and true distributions often exist on low-dimensional manifolds and are unlikely to intersect significantly. This situation can lead to an infinite KL distance, contributing to unstable training. The Wasserstein GAN \citep{Arjovsky2017Wasserstein} significantly advances the learning by minimizing an approximation of the Wasserstein distance. Drawing inspiration from the theories of Wasserstein distance and certain physical concepts, along with tools such as the Gaussian kernel covariance matrix and the Gaussian mixture model, we introduce a novel gradient approximation function named \textbf{Negative Gaussian Mixture Gradient (NGMG)}. This function is versatile and can be applied to various distance functions. We show that NGMG is a linear transform of the Wasserstein distance, inheriting its advantages as a cost function over many likelihood-based methods. We provide a distribution learning algorithm based on NGMG and conduct experiments in density estimation and neural network training. Our observations suggest that the proposed NGMG outperforms the classical binary cross-entropy loss.

Section 2 introduces preliminary concepts and establishes the general notation used throughout the paper. Section 3 details the proposed diffusion model, which conditions on the Gaussian mixture model, and includes experiments as well as a comprehensive theoretical analysis using set theory. Section 4 describes our Negative Gaussian Mixture Gradient function, presenting the methodology and a theoretical comparison between the Wasserstein distance and NGMG. Additionally, experiments are conducted to evaluate the efficacy of the NGMG approach. The final section draws conclusions from the study.


\section{Preliminary}
\subsection{Gaussian Mixture Model}

A mixture distribution, also known as a mixture model, is the probability distribution that results from a convex combination of different distributions. The probability density function (PDF), or mixture density, is typically a weighted sum of the distributions' PDFs, with strictly non-negative weights that sum to one. The Gaussian mixture model (GMM) is a type of mixture distribution that assumes all data points are generated from a mixture of a finite number of Gaussian distributions. The PDF of a GMM, as shown in Eq.\eqref{GMMdensity}, is a linear combination of the Gaussian distributions. For each weight $\pi_n \geq 0$, 
\begin{eqnarray*}
    \sum_{n=1}^{N}\pi_n =1.
\end{eqnarray*}
\begin{eqnarray}\label{GMMdensity}
    G(x)=\sum_{n=1}^{N}\pi_n \phi(x;M_n, \Sigma_n),
\end{eqnarray}
where $\phi$ is density function of Gaussian distributions, $x\in \mathbb{R}^D$, mean vector $M_n\in \mathbb{R}^D$, covariance matrix $\Sigma_n\in{\mathbb{R}^{+}}^D\times {\mathbb{R}^{+}}^D$ and $N,D\in \mathbb{N}^+$.

\subsection{GMM Expansion and Learning Algorithm for Density Estimation}
For GMM, the weights $\{\pi_n\}$, mean vectors $\{M_n\}$, and covariance matrices $\{\Sigma_n\}$ are parameterized to enable the application of learning or optimization techniques for tasks such as modeling. The expectation-maximization (EM) algorithm is the classical learning approach for GMMs. This likelihood-based method employs latent variables to cluster the observed data points, iteratively updating the parameter estimates. The EM algorithm has several known limitations, including sensitivity to initial parameter settings and the propensity to converge to local minima, as extensively discussed in the literature \citep{amendola2015maximum, abb2008, bier2003, blo2013, jin2016local, kon1992, kon1993, kwe2013, mck1994, pac2001, shi2017, srebro2007there}. When comparing the distances between two distributions, especially if they originate from wholly distinct distribution families, many methods are unsuitable for leveraging the distance function for learning purposes. For example, it is typically infeasible to approximate a mixture distribution accurately by learning a normal distribution.

A straightforward approach has been introduced \citep{Lu2023Efficient, Lu2023Efficient2} to address these challenges, making GMMs more tractable for learning and more adaptable to contemporary machine learning models, such as neural networks. This method comprises two key components: GMM expansion and a corresponding learning algorithm. \textbf{GMM expansion}, akin to the Fourier series, posits that any density can be approximately represented by a Gaussian mixture model. Under this framework, the component Gaussian distributions are defined with fixed means $M_n$ and covariances $\Sigma_n$, serving as base distributions that do not require further parameterization. The weights $\{\pi_n\}$ are the sole parameters to be learned from the data. This simplification not only facilitates the learning of densities but also standardizes the comparison of distance functions within the same base framework. An efficient one-iteration learning algorithm for this approach is detailed in \cite{Lu2023Efficient2}. The method can be summarized by the following steps:

\begin{enumerate}\label{1ter}
    \item Define $N$ Gaussian distributions and evenly spread $\{M_n\}$ across dataset $\max\left(x\right),\min\left(x\right)$,
    \item Define hyper parameters $\{\Sigma_n\}$,
    \item Initialize $\pi_1=\pi_2...=\pi_N=1/N$,
    \item Calculate $l_n=\sum_{d=1}^D\phi(x_d;M_n, \Sigma_n)=\sum_{d=1}^D\phi_n(x_d)$,
    \item $\pi_{n}^{+1}=\frac{l_n}{\sum_{n=1}^Nl_n}$.
\end{enumerate}

Informed by the above studies, any two distributions, $g_1$ and $g_2$, can be approximated by two categorical distributions, $\pi_1$ and $\pi_2$ respectively. The discrepancy between $g_1$ and $g_2$ can thus be equated to the difference between $\pi_1$ and $\pi_2$. In this work, we employ GMM to build latent distributions and utilize it as a cost function within our proposed framework. The groundwork of our approach is the concept of GMM expansion, with more detailed elucidation to be provided in Sections 2.4 and 4.

\subsection{Wasserstein-Distance and GMM Expansion}
In this work, we proposed a novel distance function called negative Gaussian mixture gradient. We discover that our method is highly relative to Wassersetein distance and share the same benefit. Methodology and proofs are provided in Section 4. Here we shows general notation of Wasserstein distance and it's special form under GMM expansion. 

Consider a measurable space ($\Omega$, $\mathcal{F}$), $P$ and $Q$ are probability measure defined on ($\Omega$, $\mathcal{F}$). The \textbf{Earth-Mover distance or Wasserstein distance} is defined by:
\begin{eqnarray*}
    W_n(P,Q)=\left(\inf_{\gamma\in\Gamma(P,Q)}E_{(x,y)\sim\gamma}d(x,y)^n\right)^{1/n},\\
    \int \gamma(x,y)dy=P,\\
    \int \gamma(x,y)dx=Q. 
\end{eqnarray*}

1-Wasserstein distance for 1 dimensional distribution:
\begin{eqnarray*}
    W_1(P,Q)=\int_{0}^{1}\mid F^{-1}_p(z)-F^{-1}_q(z)\mid dz,
\end{eqnarray*}
where $z$ is the quantile and $F^{-1}$ is the inverse  cumulative distribution function.  \\
The \textbf{$p$-Wasserstein distance} between $g_1$ and $g_2$ is regarded as 
\begin{eqnarray}
W_p(g_1,g_2)=\left(\int_{0}^{1}\mid G_1^{-1}(q)-G_2^{-1}(q)\mid^p \,dq \right)^{1/p},\label{wdistance}
\end{eqnarray}
where $g_1$ and $g_2$ are both one-dimensional distributions and $G_1$, $G_2$ are distribution functions, respectively.  Consider $G_1$ and $G_2$ have bounded support $M$. When $P =1$ in Eq.\eqref{wdistance}, the scheme of transportation $x=G_2^{-1}\circ G_1$ leads to
\begin{eqnarray}\label{sep}
    W_1(g_1,g_2)=\int_{\mathbf{M}}\mid G_1(x)-G_2(x)\mid \,dx,
\end{eqnarray}
where continuous $G_i:  Lipz(G_i)\leq1$, $i=1,2$. $G_i$ is satisfied by \textbf{Lipschitz continuity} and $ Lipz(G_i)$ is the Lipschitz constant. Based on the duality theorem of Kantorovich and Rubinstein (1958), $W_1(g_1,g_2)$ is bounded because $\mid G_1(x)-G_2(x)\mid$ is bounded. 
\begin{proposition}\label{proposition1}
    Under the GMM expansion setup, the 1-Wasserstein of two distribution is given by:
   \begin{eqnarray*}
\begin{aligned}
    W_1(\pi^{(1)},\pi^{(2)}) &=\mid \parallel\overrightarrow{W_1}(\pi^{(1)},\pi^{(2)})^+\parallel_1-\parallel\overrightarrow{W_1}(\pi^{(1)},\pi^{(2)})^-\parallel_1\mid \\
    &=\mid \parallel FB\cdot (\pi^{(1)}-\pi^{(2)})\parallel_1- \parallel  F(I-B)\cdot (\pi^{(1)}-\pi^{(2)})\parallel_1 \mid
\end{aligned}
\end{eqnarray*}
where $I$ is the identity matrix, $B$ is a matrix given by:
\begin{eqnarray*}
\begin{aligned}
    B &= \begin{bmatrix}b_1 & \dots & 0\\
    \vdots &\ddots& \vdots \\
    0 &\cdots&b_N  \end{bmatrix},~~
    b_n=\begin{cases}1 & \pi_n^{(1)}-\pi_n^{(2)} \geq 0\\0 & \pi_n^{(1)}-\pi_n^{(2)} < 0\end{cases},\\
\end{aligned}
\end{eqnarray*}
\end{proposition}
\begin{proof}
According to the GMM \eqref{GMMdensity}, we have 
\begin{eqnarray*}
    G_i(x)= \int_{-\infty}^x\sum_{n=1}^N  \pi_n^{(i)}\phi_n(s) ds= \sum_{n=1}^N \pi_n^{(i)}\int_{-\infty}^x\phi_n(s)ds=\sum_{n=1}^N  \pi_n^{(i)}F_n(x),\quad i=1,2.\notag
\end{eqnarray*}
Substituting it into \eqref{sep},  we get
\begin{eqnarray}\label{w1}
\begin{aligned}
    W_1(g_1,g_2)
    &=\int_{\mathbf{M}}\mid \sum_{n=1}^N ( \pi_n^{(1)}F_n(x)-\pi_n^{(2)}F_n(x))\mid \,dx\\
    &=\int_{\mathbf{M}}\mid \sum_{n=1}^N (\pi_n^{(1)}-\pi_n^{(2)}) F_n(x)\mid \,dx.
\end{aligned}
\end{eqnarray}

Under GMM expansion, the Wassertein distance between $g_1$ and $g_2$ becomes a function of $\pi_n^{(1)}$ and $\pi_n^{(2)}$. 

Let $\pi^{(i)} =[\pi_1^{(i)},\dots\pi_N^{(i)}]^T,i=1,2$, $W_1(g_1,g_2)=W_1(\pi^{(1)},\pi^{(2)})$, 
\begin{eqnarray}\label{gmmwasdist}
\begin{aligned}
    W_1(\pi^{(1)},\pi^{(2)})^+ &=\int_{\mathbf{M}} \sum_{n=1}^N b_n(\pi_n^{(1)}-\pi_n^{(2)})  F_n(x) \,dx, \\
    ~\\
    W_1(\pi^{(1)},\pi^{(2)})^- &=\int_{\mathbf{M}} \sum_{n=1}^N (1-b_n)(\pi_n^{(1)}-\pi_n^{(2)})  F_n(x) \,dx, \\
    ~\\
    W_1(\pi^{(1)},\pi^{(2)}) &=\mid 
    W_1(\pi^{(1)},\pi^{(2)})^++W_1(\pi^{(1)},\pi^{(2)})^-\mid.
\end{aligned}
\end{eqnarray}

Next, we can vectorize Eq.\eqref{gmmwasdist} and get the following presentation:
\begin{eqnarray*}
\begin{aligned}
    \overrightarrow{W_1}(\pi^{(1)},\pi^{(2)})^+&=\begin{bmatrix}b_1(\pi_1^{(1)}-\pi_1^{(2)})  \int_{\mathbf{M}}  F_1(x) \,dx  \\ \vdots \\b_N(\pi_N^{(1)}-\pi_N^{(2)})\int_{\mathbf{M}}    F_N(x) \,dx \end{bmatrix}=F\begin{bmatrix}b_1(\pi_1^{(1)}-\pi_1^{(2)})    \\ \vdots \\b_N(\pi_N^{(1)}-\pi_N^{(2)})\end{bmatrix}\\
    &=FB\cdot(\pi^{(1)}-\pi^{(2)}),\\
    ~~\\
    \overrightarrow{W_1}(\pi^{(1)},\pi^{(2)})^-&=\begin{bmatrix}(-b_1)(\pi_1^{(1)}-\pi_1^{(2)})  \int_{\mathbf{M}}  F_1(x) \,dx  \\ \vdots \\(1-b_N)(\pi_N^{(1)}-\pi_N^{(2)})\int_{\mathbf{M}}    F_N(x) \,dx \end{bmatrix}\\
    &=F(I-B)\cdot (\pi^{(1)}-\pi^{(2)}),\\
    ~~\\
    F &=\begin{bmatrix}\int_{\mathbf{M}} F_1(x) \,dx & \cdots &0 \\
    \vdots&\ddots&\vdots\\
     0&\cdots & \int_{\mathbf{M}} F_N(x) \,dx\end{bmatrix}, \\
\end{aligned}
\end{eqnarray*}
where $I$ is the identity matrix.
Under GMM expansion, we can rewrite Eq.\eqref{w1} as follows:
\begin{eqnarray}\label{gmmwasdist2}
\begin{aligned}
    W_1(\pi^{(1)},\pi^{(2)}) &=\mid \parallel\overrightarrow{W_1}(\pi^{(1)},\pi^{(2)})^+\parallel_1-\parallel\overrightarrow{W_1}(\pi^{(1)},\pi^{(2)})^-\parallel_1\mid \\
    &=\mid \parallel FB\cdot (\pi^{(1)}-\pi^{(2)})\parallel_1- \parallel  F(I-B)\cdot (\pi^{(1)}-\pi^{(2)})\parallel_1 \mid
\end{aligned}
\end{eqnarray}
\end{proof}
 Eq.\eqref{gmmwasdist2} shows a vectorized representation of Wasserstein distance under GMM expansion in bounded support $\mathcal{M}$. This will be used in Section 4 to show that the proposed negative Gaussian mixture gradient function is strongly related to Wasserstein distance.

\section{Diffusion Model Conditioning on Gaussian Mixture Model}
Neural networks are deterministic systems. In a probabilistic perspective, a neural network can be seen as an objective function that is trained to map the event space of a latent distribution onto a subset of real space. In most cases, probabilistic events do not originate from the model itself; they typically arise from latent variables. One of the most well known example is the GAN. In classic GANs, latent variables are drawn from normal distributions. Neural network as a objective function, the output of the GAN is expected be comparatively normal, with the mean and variance projected into a high-dimensional, non-linear form. Mode collapse is an inherent issue in this structure due to the limitations of Gaussian latents. In this work, we focus on the probabilistic interpretation of neural networks. We propose a conditioning mechanism using a Gaussian mixture model and provide a theoretical analysis of the connection between latent and real distributions later in this section.

\subsection{Model Architecture}
In a text-to-image model, text information provides insights into how we humans interpret images. In other words, images are conditioned by text, and text serves as an embedded variable within a neural network to guide the generation process. For example, given the input `A human with glasses',w a language model will process this input and produce a corresponding output for the generation. However, intuitively, `A human with glasses' is more likely to represent a concept or a distribution than a specific value. In our model, we propose that a feature/text/concept should be considered distributed instead of a fixed value. While the diffusion model assigns a latent value through a Gaussian process to approximate the data distribution, this remains a debatable assumption. The image distribution for `A human with glasses' may not be centered around any mean image. Rather, it is diverse and widespread and potentially exhibits infinite variations. Based on the preceding rationale, we suggest that features are intrinsically distributional, which informs the setup of our model. We represent the distribution with $D(\cdot)$, and our model's mathematical formulation is expressed as follows:
\begin{eqnarray}\label{hypoth1}
    \begin{aligned}
        Target &=NN(\mathcal{Z})+\epsilon,\\
        \mathcal{Z} &=F_{L}\circ F_{L-1} \circ \cdots F_{1}(\mathcal{X}),\\
        \mathcal{X} &=[X_1,X_2,...X_K]^T,\\
        X_k &\sim D(\theta_k),\\
    \end{aligned}
\end{eqnarray}
where $F_1, F_2, \cdots, F_L$ are functions. They can be any function that is used as an objective function. 
\begin{figure}[ht]
	\centering
		\includegraphics[scale=0.45]{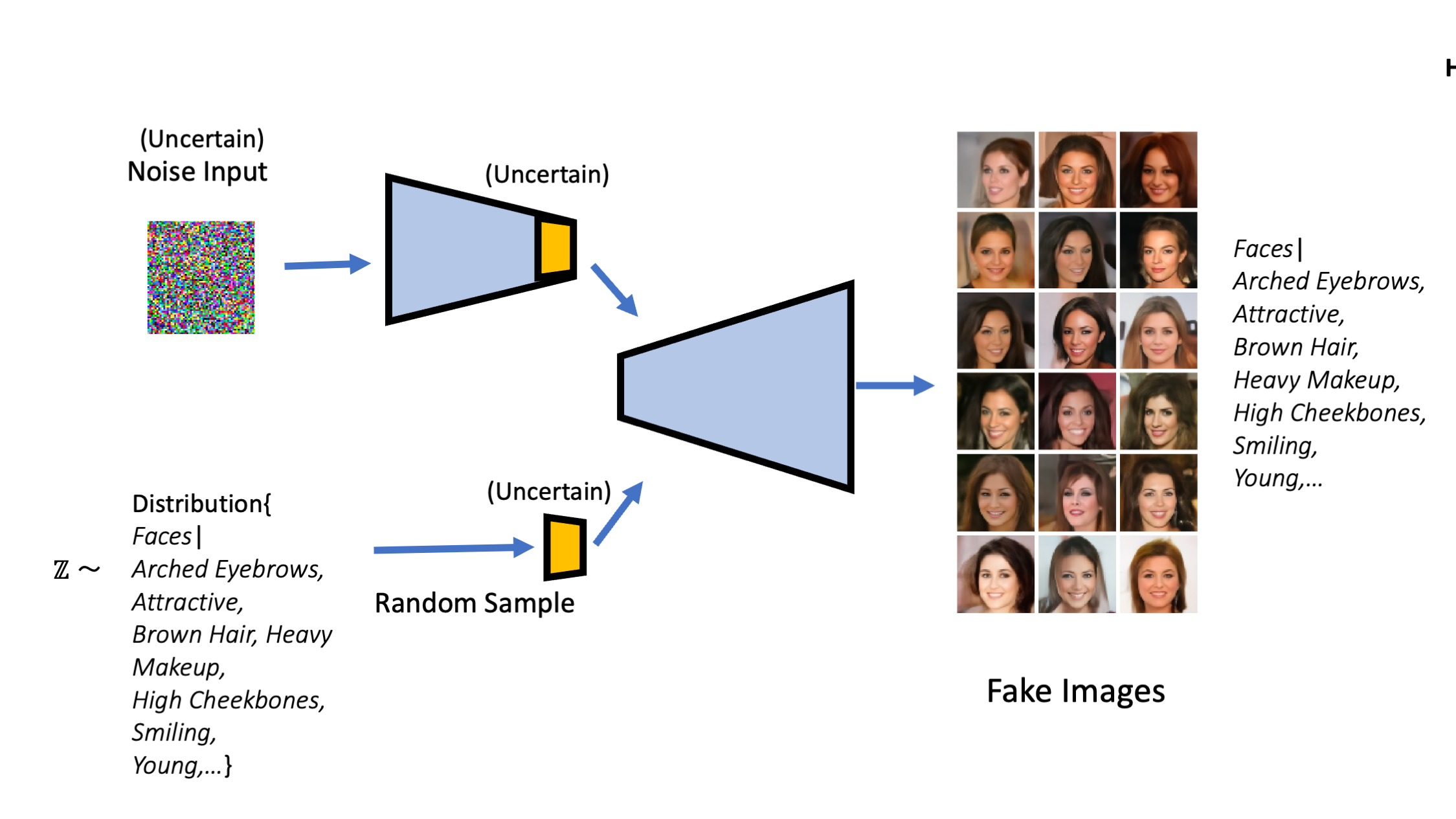}
	\caption{Image generation used conditional feature distribution. When model is given conditional distribution of $\mathcal{Z}$, model has two source of uncertainty. One is from the denoise Gaussian process. Second is the conditional distribution.}
	\label{picFacesampling}
\end{figure}

This model setup is indicative of a hierarchical feature reconstruction that is consistent with the fundamental principles of neural networks. Functions $F_l$ transform lower-level feature distributions into higher-level ones. Our proposed diffusion model is depicted in Figure \ref{picFacesampling} and Figure \ref{picFaceSampling1}. The most notable distinction between our model and other diffusion models lies in the treatment of the latent space $\mathcal{Z}$. As illustrated in Figure \ref{picFacesampling}, to generate images from specific feature categories, we sample from the conditional distribution of $\mathcal{Z}$ randomly. Given a particular latent value within $\mathcal{Z}$, our model can produce a series of samples that resemble target images but exhibit subtle detail variations. Such variations, which include attributes like hair texture, earrings, and skin color, resemble Gaussian distributions characterized by a mean and small variance, as exemplified in Figure \ref{picFaceSampling1}.

The architecture utilize two types of uncertainty to approximate a subset of the real space. The first is the latent distribution of $\mathcal{Z}$, and the second is the application of a Gaussian process, which introduces Gaussian noise.

\begin{figure}[ht]
	\centering
		\includegraphics[scale=0.45]{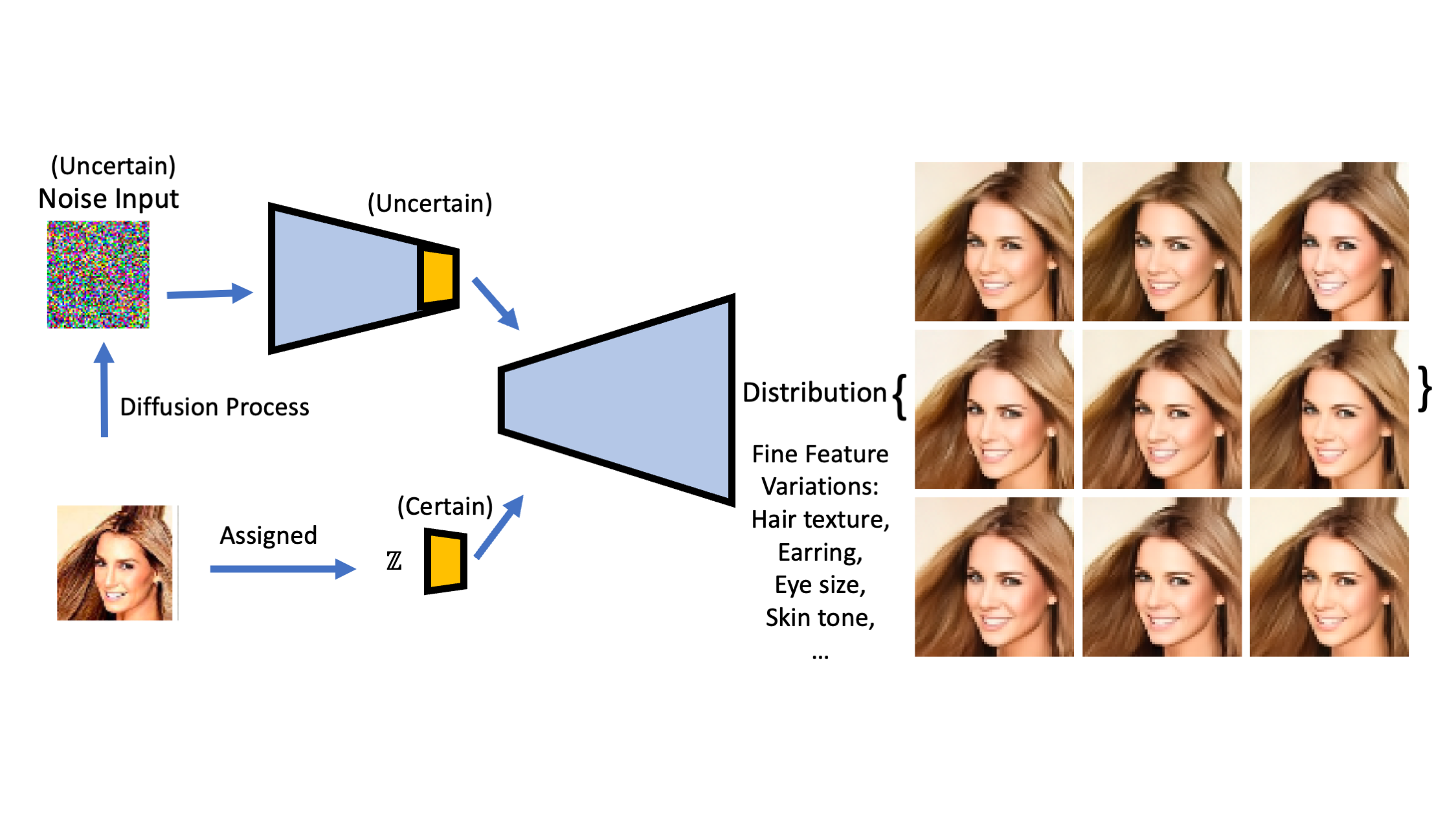}
	\caption{Image generated uses fixed certain $\mathcal{Z}$. When a particular sample of $\mathcal{Z}$ is given, uncertainty only coming from denoise Gaussian process. Each training image is assigned a sample value from $\mathcal{Z}$.}
	\label{picFaceSampling1}
\end{figure}

\ref{Appendix1} provides random samples from the trained models. \ref{Appendix2} details the sampling method used to obtain our results. Figure \ref{picx0Sampling} illustrates the generation results at each step. We use a comparatively larger $\beta$ and a total of 100 diffusion steps to train our diffusion model on the CelebA dataset \citep{Liu2015Face}. The backward autoregressive denoising process starts at step 100 with complete Gaussian noise. The results indicate that our model rapidly converges to the target images and refines details at each sampling step. The predicted $x_0^*$ from steps 80 to 100 exhibit similar generation quality, which is in line with our expectations based on the model's design. The latent distribution of $\mathcal{Z}$ primarily accounts for feature variations on a larger scale, while the Gaussian process is responsible for the finer variations, consistent with its 'normal' properties.

\begin{figure}[ht]
	\centering
		\includegraphics[scale=0.4]{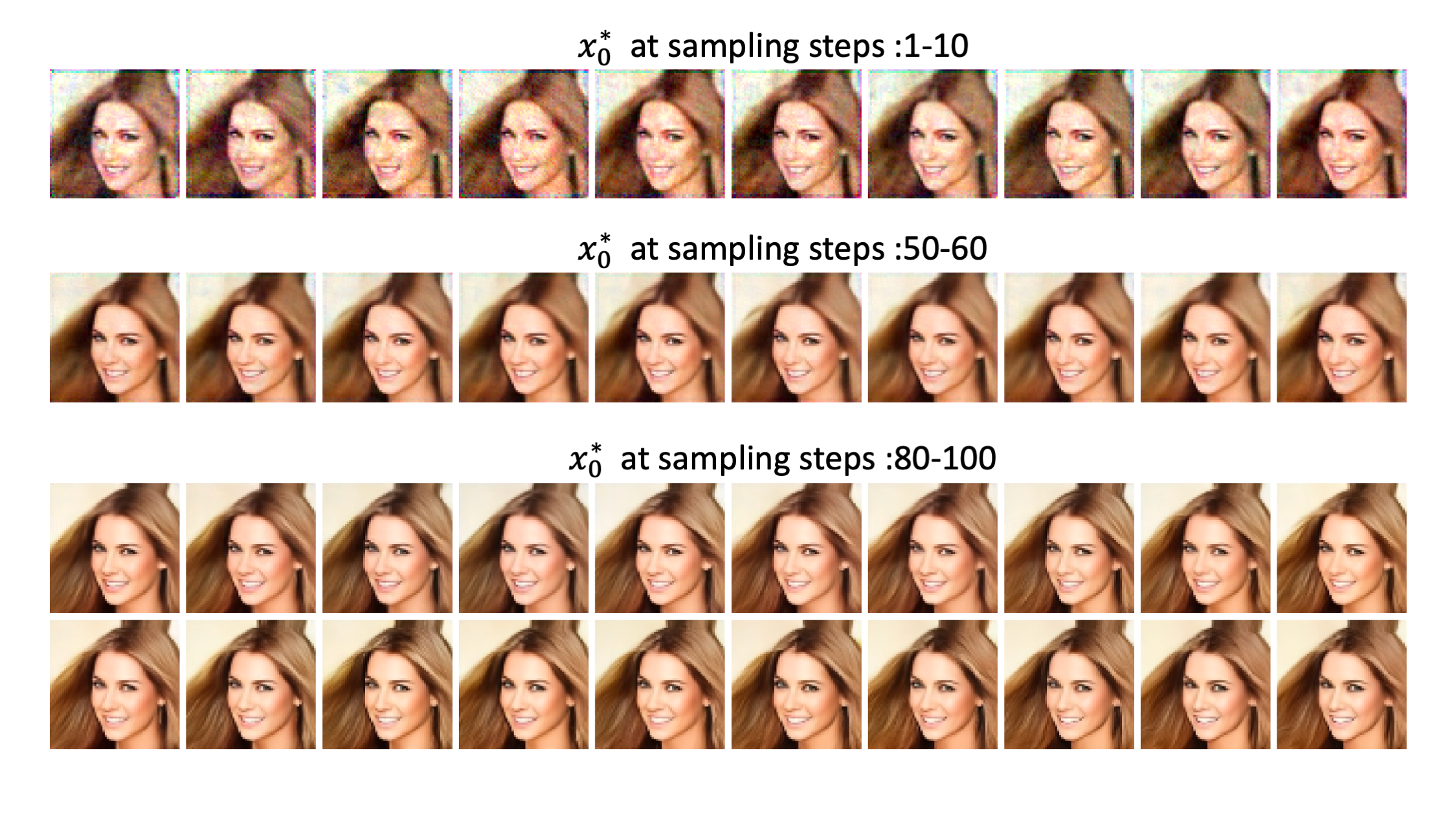}
	\caption{Our sampling process is constantly going back to $x_0$. More detail given in \ref{Appendix2}. $x_0^*$ is the predicted $x_0$. $\beta$ is the variable control the size of noise in diffusion process. The sampling process is shown by Eq.\eqref{denoise}. Our model is trained with comparably larger $\beta$ and only take $T=100$ diffusion steps. Denoisng process start with $x_T$ as a Gaussian noise.}
	\label{picx0Sampling}
\end{figure}

\subsection{Design of Latent $\mathcal{Z}$}

In contrast to text-to-image models that employ a language encoder for conditioning, our method utilizes a specially designed set of latent variables, $\mathcal{Z}$, to directly guide the denoising process. This allows $\mathcal{Z}$ to be trained to capture a subset of the real distribution. In auto-encoders and most of their variations, latent variables usually refer to the encoder output, also known as the bottleneck layer. An interesting example is VQ-VAEs, which map latent variables to a discrete codebook, revealing that neural network latents are likely not continuous. Moreover, the model setup in Eq.\eqref{hypoth1} implies that if the lower-level features $X$ are sampled from a Gaussian distribution, the data distribution needs to be close to a normal distribution centered on some mean. This is due to the additive property of normal distributions. If given some normally distributed independent random variables, their sum is still normally distributed. In other words, the functions $F_l$ in Eq.\eqref{hypoth1} are linear transformations $F_l:\mathcal{W}_{l} \rightarrow \mathcal{W}_{l+1}$, where $\mathcal{W}_{l}$ are normally distributed, and if we apply the ReLU non-linearity to $F_l$, $\mathcal{W}_{l+1}$ is still following some form of a single-peak distribution. Studies report that GMMs are beneficial for various tasks compared to conventional methods such as embedding and learning latent distributions \citep{Nachmani2021Non,Kolouri2018Sliced,Lu2023Efficient,Lu2023Efficient2}. These findings motivate us to use GMM to embed latent variables. GMM can be seen as a combination of discrete and continuous distributions. It not only has all the properties of a normal distribution but also those of a discrete distribution. Our model setup in Eq.\eqref{hypoth1} suggests that instead of embedding text into a fixed value, using a distribution is preferable. For $n=1,2,\dots,N$ and $k\in[1,K]\cap \mathbb{N}$,

\begin{eqnarray}\notag
    A=[A_1,A_2,\dots,A_{K}]^T
\end{eqnarray}
and
\begin{eqnarray}\notag
A_k=
    \begin{cases}1 & \text{if image has feature $k$}\\0 & \text{otherwise}\end{cases}.
\end{eqnarray}

\begin{eqnarray}\label{smallz}
    Z_k=\begin{bmatrix}z_1 \\.\\.\\.\\z_N \end{bmatrix},\quad z_n\sim GMM(\pi,\mu,\sigma) .
\end{eqnarray}

Latent variable vector $\mathcal{Z}$ is defined as follows:
\begin{eqnarray}\notag
    \mathcal{Z}=[Z_1A_1,Z_2A_2,Z_3A_3,\dots ,Z_{K}A_{K}].
\end{eqnarray}

Under this configuration, $\mathcal{Z}$ adheres to a conditional distribution contingent on a given set of features. The indicator function $A$ determines whether a feature is active or inactive. In practice, the $z_n$ values are sampled from the Gaussian mixture model and are coupled with each data image for supervised learning. Traditional GANs utilize a Gaussian distribution, however, within our framework, a Gaussian mixture is employed as suggested by Eq.\ref{smallz}, which delineates the representation for each feature. Empirical studies indicate that the Gaussian mixture model \citep{Lu2023Efficient, Lu2023Efficient2} outperforms a singular Gaussian approach. Although a mathematical rationale for this phenomenon is not provided, our model, as described in Eq.\eqref{hypoth1}, suggests that lower-level features such as hair length—categorized as [`Long', `Short', `Median']—are more suitably represented by a discrete distribution \citep{Razavi2019Generating, VanDenOord2017Neural} or a continuous multi-modal distribution. In our GMM setup, we use three Gaussian components with uniform $\pi$ values. The variance $\sigma$ is constant across all components, and the mean $\mu$ of each component is sufficiently spaced such that the distance between each $\mu$ larger than $\sigma$.

In terms of conditioning, we can condition on classes or on features. Class information is usually mutually exclusive or disjoint, whereas features are not. For example, text can be considered a type of feature information. Each image is labeled a set of words or a sentence instead of a single class. We have found that using feature information o of class information significantly improves model performance. A theoretical explanation based on set theory is provided in the next subsection.
 
\subsection{Conditioning on Feature or Conditioning on Class?}

Conditioning mechanisms in neural networks are directly related to latent distribution. A change in latent distribution could cause performance to be drastically different under the same model specifications. Denote that the event space of data distribution is $S^d$ and it is a subset of our real target distribution event space $S^\mathcal{R}$. Data distribution may be conditional based on some rules, but it is not always necessary. The event space of the data distribution could be just random samples. Consider we have latent variables $S^\mathcal{Z}$. If all the data is given class information, which consists of $M_1$ independent classes, we will split $S^\mathcal{Z}$ into $\{S^{\mathcal{Z}_m}\}_{m=1,2,\dots,M_1}$. If a dataset is labeled with features of size $N$, the situation becomes more complex. The dataset is not categorized into N features; instead, it is split into fine sub-spaces, which are defined by the intersection of certain features. Assume there are $M_2$ sets of intersection subsets, and elements in each subset are $\{\mathcal{Z}_{n,m}\}_{n\in \mathbb{N}^+,m=1,2,\dots,M_2}$. In other words, latent event spaces for features are as follows:
\begin{eqnarray*}\notag
    \mathcal{Z}_{1,m} \cap \mathcal{Z}_{2,m}\cap \dots \cap \mathcal{Z}_{n,m}\cap\dots=\mathcal{Z}_m^*.
\end{eqnarray*}

These defined spaces also imply that:

\begin{eqnarray*}\notag
\bigcup_1^{M_2}S^{\mathcal{Z}^*_m}\subseteq S^{\mathcal{Z}}, \ \ \bigcup_1^{M_1}S^{\mathcal{Z}_m}\subseteq S^{\mathcal{Z}}
\end{eqnarray*}

Neural network is the objective function that takes latent space $S^{\mathcal{Z}_m}$, $S^{\mathcal{Z}_m^*}$ to $\{S^{NN(\mathcal{Z}_m)}\}_{m=1,2,\dots,M_1}$, 
$\{S^{NN(\mathcal{Z}_m^*})\}_{m=1,2,\dots,M_2}$. It is reasonable to assume that a neural network has some high-dimensional error $\epsilon$. We can define that $S^{NN(\bullet)+\epsilon}$ is a direct sum of two subsets:
\begin{equation}\notag
S^{NN(\bullet)+\epsilon}=S^{NN(\bullet)+\epsilon}_\mathcal{R}\cup S^{NN(\bullet)+\epsilon}_\mathcal{D},
\end{equation}
where $S^{NN(\bullet)+\epsilon}_\mathcal{R}$ is the subset of generations that are considered within the event space in real distribution. $S^{NN(\bullet)+\epsilon}_\mathcal{D}$ is the subset of bad generation that is out of the real space. They satisfy
\begin{eqnarray*}
    S^{NN(\bullet)+\epsilon}_\mathcal{D}\cap S^{\mathcal{R}}=\emptyset
\end{eqnarray*}

Assuming that the neural network is well trained, we obtain the following relationship:
\begin{eqnarray}\notag
    S^{d}\subset \bigcup_{m=1}^{M_2} S^{NN(\mathcal{Z}_m^*)+\epsilon}_{\mathcal{R}} \subseteq \bigcup_{m=1}^{M_1} S^{NN(\mathcal{Z}_m)+\epsilon}_{\mathcal{R}} \subset
    S^{\mathcal{R}}.
\end{eqnarray}
Furthermore, 
\begin{eqnarray*}
    \bigcup_{m=1}^{M_1} S^{NN(\mathcal{Z}_m)+\epsilon}-S^{\mathcal{R}}=\bigcup_{m=1}^{M_1} S^{NN(\mathcal{Z}_m)+\epsilon}_{\mathcal{D}}=S^{NN(\mathcal{Z})+\epsilon}_{\mathcal{D}},\\
    \bigcup_{m=1}^{M_2} S^{NN(\mathcal{Z}_m^*)+\epsilon}-S^{\mathcal{R}}=\bigcup_{m=1}^{M_2} S^{NN(\mathcal{Z}_m^*)+\epsilon}_{\mathcal{D}}=S^{NN(\mathcal{Z}^*)+\epsilon}_{\mathcal{D}}.
\end{eqnarray*}

Given $\text{card}(S^\mathcal{Z^*})=\text{card}(S^\mathcal{Z})$, cardinality of $S^{NN(\mathcal{Z})+\epsilon}_{\mathcal{D}}$ and $S^{NN(\mathcal{Z}^*)+\epsilon}_{\mathcal{D}} $ has the following property:

\begin{eqnarray}\label{assumption3}
\begin{aligned}
    \text{card}(S^{NN(\mathcal{Z}^*)+\epsilon}_{\mathcal{D}}) &\le \text{card}(S^{NN(\mathcal{Z})+\epsilon}_{\mathcal{D}}), \\
    \frac{\text{card}(S^{NN(\mathcal{Z}^*)+\epsilon}_{\mathcal{R}})}{\text{card}(S^{NN(\mathcal{Z}^*)+\epsilon})} &\geq \frac{\text{card}(S^{NN(\mathcal{Z})+\epsilon}_{\mathcal{R}})}{\text{card}(S^{NN(\mathcal{Z})+\epsilon})}
\end{aligned}
\end{eqnarray}

The deduction above shows that there are fundamental differences in class and feature information. Figure .\ref{cfset} shows a simple graphic explanation of Eq.\eqref{assumption3}. This mock example showcases a simple scenario with four classes and four features. Red boxes represent the input domain of the neural network. Although each subset of $\mathcal{Z}^*_m$ is significantly smaller than $\mathcal{Z}_m$, when the size of $\mathcal{Z}^*_m$ gets larger and larger, we can recover more and more event space from the real distribution. Intuitively, generations' quality is improved by limiting training within many refined subsets. Hence, based on set theory, we conclude that the subset of defect generation with latent distribution condition on features is smaller or at least equal to the subset of classes. In addition, the deduction above also provides an explanation of why GANs usually report having mode-collapse. The latent variables of GAN are usually normally distributed. Normal distributions are too centred and possibly cause models to converge into a very small subset. 

\begin{figure}[ht]
	\centering
		\includegraphics[scale=0.45]{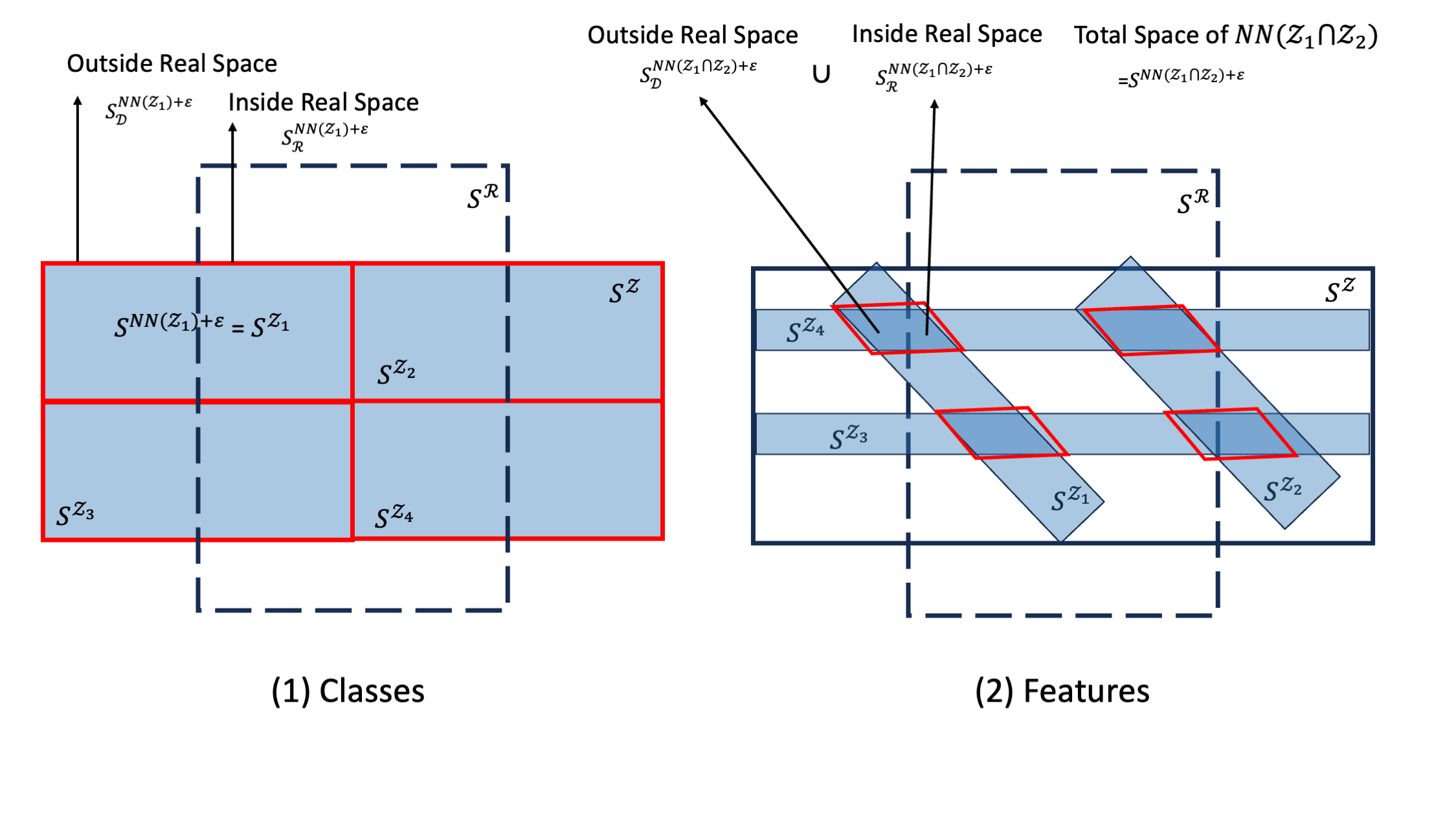}
	\caption{Probability event space of class and feature. Left: The latent event space condition on Class. Right: The latent event space condition on feature.}
	\label{cfset}
\end{figure}

Through the analysis based on set theory, the addressed issue of features and classes also needed to be experimentally tested. Cifar10 \citep{Krizhevsky2009Learning} and CelebA are used to train the diffusion model tso carry out the experimental comparison. Figure .\ref{cifarceleb} shows our results and model specifications. With similar model sizes and the same training method, the model trained on Cifar clearly underperformed the model trained on CelebA. Lost from Cifar10 is also larger than CelebA. The only difference between these two models is our latent distributions. Latent distributions of CelabA consist of features defined by the intersection of smaller subsets. In contrast, Cifar10 only provides independent class information, so latent spaces are disjoint from each other. Orange square boxes in Figure .\ref{cifarceleb} are original images, and other subplots are image reconstruction by diffusion model. Noting that the model trained on CelebA is built with fewer parameters and fewer diffusion steps. The diffusion model output for CelebA has a higher image dimension of 64x64x3 compared to 32x32x3 in Cifar10. In summary, latent distribution built by features shows a capability that uses less but produces more. 

\begin{figure}[ht]
	\centering
		\includegraphics[scale=0.45]{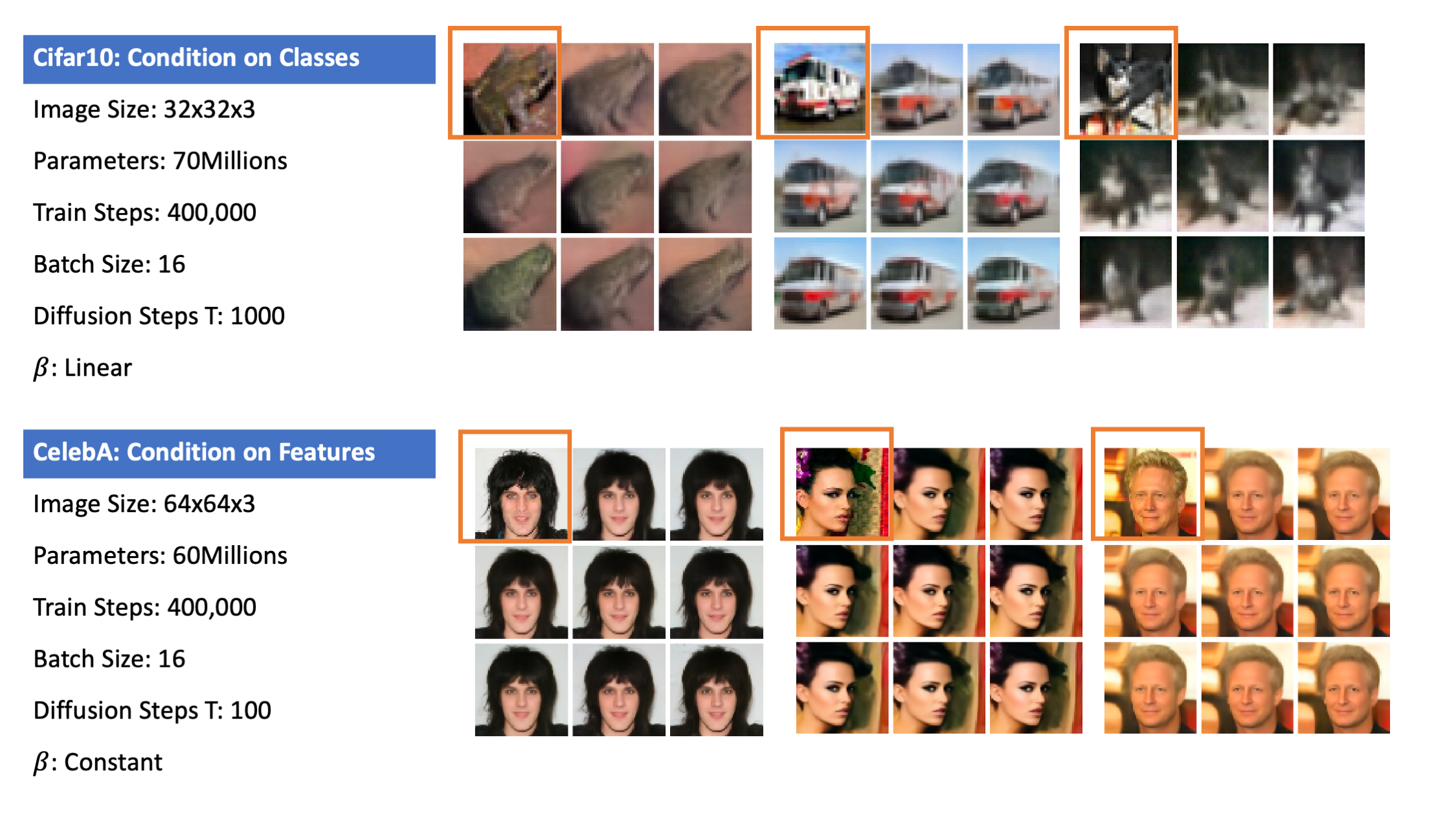}
	\caption{Features vs.Classes. Top: Images and model specification of proposed model train on Cifar10. Down: Images and model specification of proposed model train on CelebA. Orange box is the true images, and the rest of the subplots are generated samples. }
	\label{cifarceleb}
\end{figure}
\subsection{Additional Classifier}
Incorporating classification information has been shown to significantly improve model performance \citep{Mirza2014Conditional,Brock2018Large,Dumoulin2017A,Miyato2017Cgan,Dhariwal2021Diffusion}, even when using synthetic labels \citep{Lucic2019High}. Based on our previous experiments, integrating a classification network with a generative model typically yields benefits, particularly in latent space representation. Models such as Generative Adversarial Networks (GANs) extensively leverage classification to enhance performance \citep{Mirza2014Conditional, Brock2018Large, Dumoulin2017A, Miyato2017Cgan, Lucic2019High}. Dhariwal et al. \citep{Dhariwal2021Diffusion} introduced two classifier-guided sampling techniques in their diffusion model. Motivated by these findings, we incorporated classifiers at the bottleneck layer of our diffusion model.

\begin{figure}[ht]
	\centering
		\includegraphics[scale=0.45]{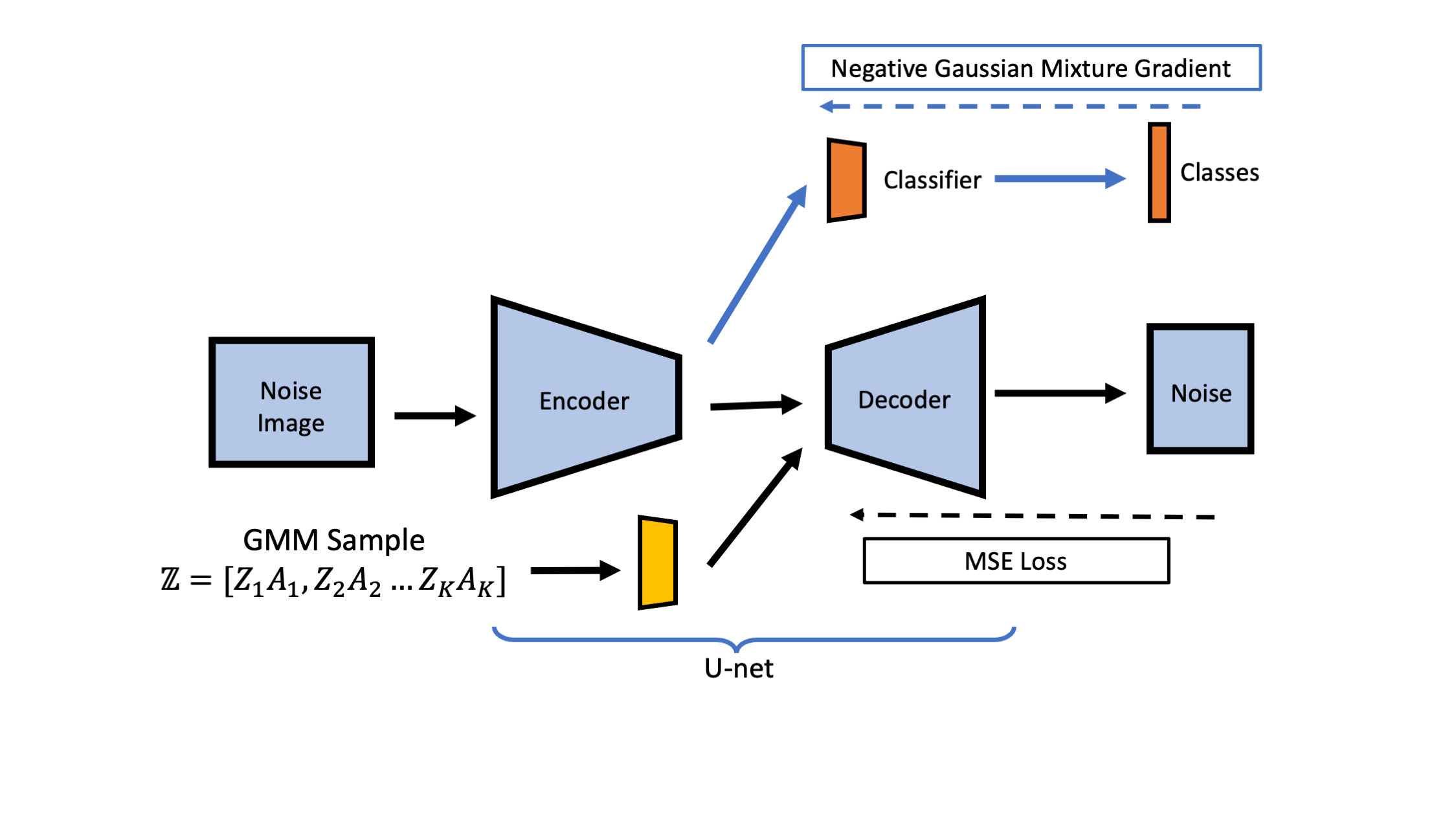}
	\caption{Proposed model with with Classifier. The difference compare to model proposed in Section 3 is added a classifier in the model. Also a novel lost function is applied in training this classifier.}
	\label{GDMsClass}
\end{figure}

Figure \ref{GDMsClass} illustrates our model architecture with an integrated classifier. Adopting a similar approach to our GMM-conditioned diffusion model, a classifier is introduced subsequent to the encoding process. We experience better stability in training which consistent with findings reported in \cite{Mirza2014Conditional, Brock2018Large, Dumoulin2017A, Miyato2017Cgan, Lucic2019High}. 19). Additionally, in terms of classifier training, rather than employing binary cross-entropy, we propose a novel distance function for training this classifier, termed the \textbf{Negative Gaussian Mixture Gradient (NGMG)}. The subsequent section will provide more details on this innovative method.

\section{Negative Gaussian Mixture Gradient}
The Negative Gaussian Mixture Gradient (NGMG) is inspired by the Wasserstein distance, the Gaussian kernel, and certain concepts in physics, such as gravitational force. Analogous to the physical world, where moving an object from point A to point B involves not only the spatial displacement but also accounts for the cost of time and the distance traveled. Similarly, NGMG conceptualizes the transportation of probability mass in a manner that encompasses more than just the difference in probability density, it includes considerations analogous to the temporal and spatial costs in the physical movement.

\subsection{Methodology of NGMG}
A consistent set of notations and clearly define the problems as follows:

\begin{itemize}
    \item $\pi$: $\pi=[\pi_1,\pi_2,...\pi_n]$ is the vector, which represents the earth in the moving problem and probability in distribution. $\pi_i^{(1)}=[\pi_1^{(1)},\pi_2^{(1)},...\pi_n^{(1)}]$ is the distribution we starting with and $\pi_i^{(2)}=[\pi_1^{(2)},\pi_2^{(2)},...\pi_n^{(2)}]$ is the target distribution. 
    \item For most distribution, GMM expansion method shown in Section 2.2 can be applied to approximate it with a set of $\{\pi_i\}$. 
    \item  $\mu,\sigma, \phi_n(x)$ are set-up by GMM expansion in Eq.\eqref{GMMdensity}. Each $\pi_i$ is assigned a $\mu_i$ which gives the object a distance.  Each $\mu_i$ is assigned a $\phi_i$ which is applied to calculate a modified Gaussian kernel covariance matrix. 
    
\end{itemize}

Define a function $L=\pi^{(1)}-\pi^{(2)}$ and only take the negative value of the function.
\begin{equation}\label{functionL}
    L_i=\begin{cases}\pi_i^{(1)}-\pi_i^{(2)} &\quad \pi_i^{(1)}-\pi_i^{(2)} < 0\\0 & \quad \pi_i^{(1)}-\pi_i^{(2)} \geq 0\end{cases}.
\end{equation}

Because
\begin{eqnarray}\notag
    \sum_{i} \pi_i=1,\\
    \mathrm{MIN}(\pi^{(1)}) - \mathrm{MAX}(\pi^{(2)})=-1,
\end{eqnarray}
we have
\begin{equation}\notag
    -1\leq\sum_{i=1}^n L_i\leq0.
\end{equation}
The $\mathrm{NGMG}$ for $\pi_j$ is:
\begin{equation}\label{eqngmg}
    ngmg(\pi^{(1)}_j,\pi^{(2)}_j)=\sum_{i\neq j}^n L_i\phi_i(\mu_j;\mu_i,\sigma),i \neq j
\end{equation}
where $i=1,2,3...n, \text{and } i\neq j$, $\sigma$ is a hyper-parameter. Rewrite Eq.\eqref{eqngmg} as matrix multiplication:

\begin{eqnarray}\label{ngmgfinal}
\begin{aligned}
    \mathit{M}&=\begin{bmatrix}0 & \phi_2(\mu_1;\mu_2,\sigma) & \phi_3(\mu_1;\mu_3,\sigma) &...\phi_n(\mu_1;\mu_n,\sigma)\\\phi_1(\mu_2;\mu_1,\sigma) &0 &\phi_3(\mu_2;\mu_3,\sigma)&... \phi_n(\mu_2;\mu_n,\sigma)\\
.&. &.&.\\
.&. &.&.\\
.&. &.&.\\
\phi_1(\mu_n;\mu_1,\sigma) &\phi_2(\mu_n;\mu_2,\sigma) &\phi_3(\mu_n;\mu_3,\sigma)&... 0\\
\end{bmatrix},\\
\text{NGMG}(L)&=-\mathit{M}\mathit{L}
\end{aligned}
\end{eqnarray}

where $\mathit{L}$ is the vector of $L_i$. $\mathit{M}(\sigma)$ is a Gaussian kernel with diagonal entries set to be zeros. Setting diagonal entries to zeros because we want the gradient of $\pi_i$ does not depend on $L_i$. The essence of Eq.\eqref{ngmgfinal} is energy and gravitation. $L_i$ is the value of the energy at position $i$ and spread out by a Guassian kernel. This energy is negative and attracts other positive value to fill in. $\sigma$ controls how far the $L_i$ could affect. At each position $i$, it moves (gradient) based on the sum of all negative energy generated by other positions and is not affected by its' own. When all negative energy is filled, no gradient will be produced. This gradient function is bounded by:

\begin{eqnarray}\notag
    \begin{aligned}
        \sup(NGMG(L))&=0,\\
        \inf(NGMG(L))&=-1.0*\max(\phi(\mu_j;\mu_i,\sigma)),
    \end{aligned}
\end{eqnarray}

\begin{figure}[ht]
	\centering
		\includegraphics[scale=0.45]{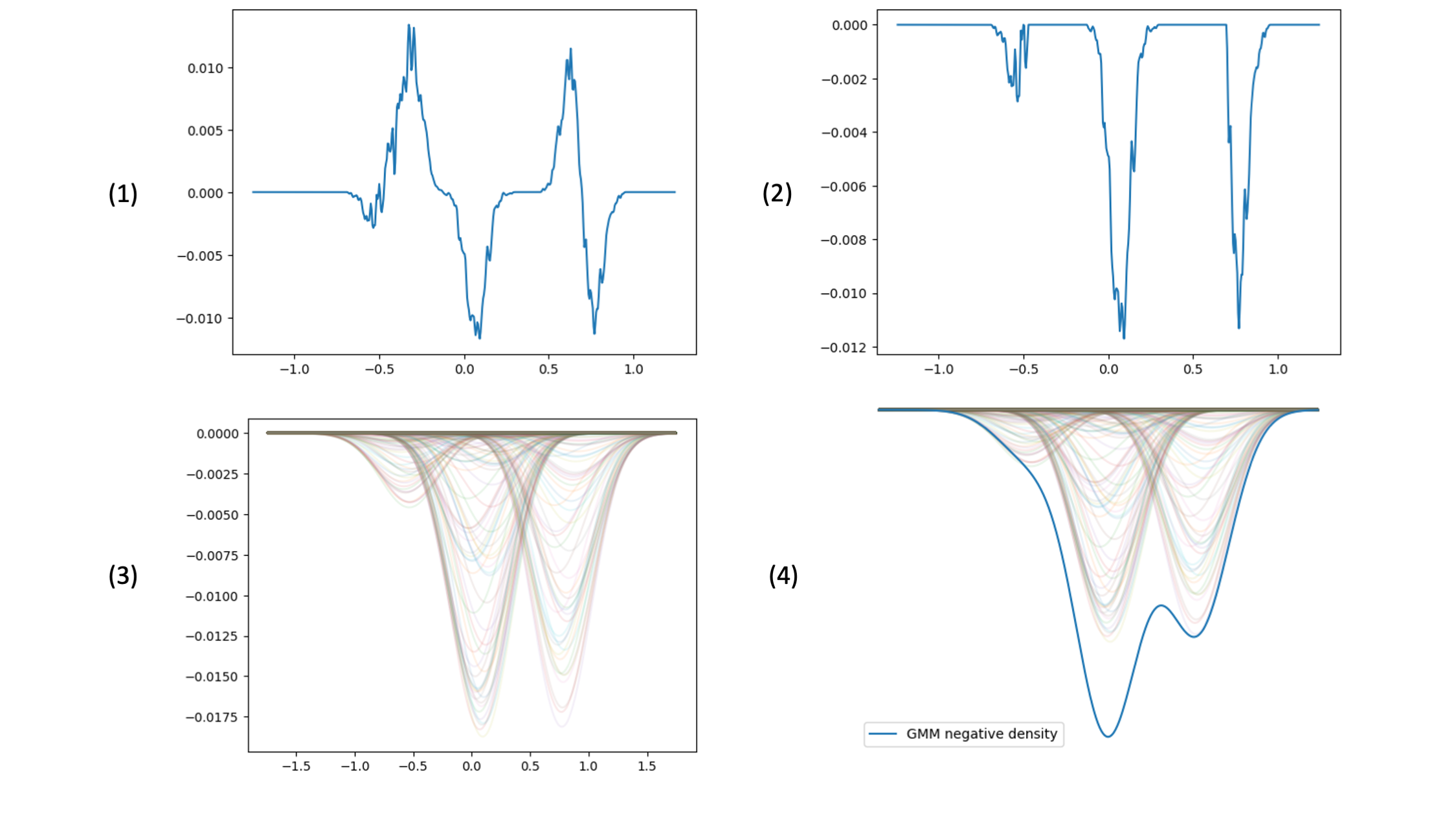}
	\caption{Illustration of Negative Gaussian Mixture Gradient(NGMG). (1) and (2) Shows the Eq.\eqref{functionL}. (3) and (4) Shows the Eq.\eqref{ngmgfinal}. x-axis is the distance($\mu$s) given to two sets of $\pi$s. The negative differences of two $\pi$s is smooth out by proposed modified covariance matrix to act as gradient for the learning. }
	\label{GMMgradient}
\end{figure}

Figure .\ref{GMMgradient} gives a stepwise demonstration to Eq.\eqref{ngmgfinal}. Subplot (1) is the graph of $\pi_i^{(1)}-\pi_i^{(2)}$ and subplot (2) is the graph taking only the negative part of $(\pi_i^{(1)}-\pi_i^{(2)})^-$, Eq.\eqref{functionL}. By assigning each $\pi$ a Gaussian distribution, a mixture negative mixture density like function is produce which shows in subplot(4). Our method utilize this negative mixture density like function as a kernel function smooth function $L$ and gives the gradient for each $\pi$.

\subsection{NGMG and Wasserstein Distance}
 \cite{Arjovsky2017Wasserstein} show that Wasserstein distance in terms of cost function provides better stability and is more sensible for training, especially in neural networks. NGMG is not a likelihood-based function by design. It is close to the Wasserstein distance and shares similar qualities in terms of cost function. In \cite{Arjovsky2017Wasserstein}, they have proven the Wasserstein distance is more sensible than Jensen-Shannon (JS) divergence, The Kullback-Leibler (KL) divergence, and Total variation (TV) distance in several aspects. In this work, we shows that Wasserstein distance can be structured by NGMG and NGMG share the same benefit in terms of cost function. Noticing that In Section 2.3, we have shown that under our GMM expansion system, any distribution $g$ is approximated by a categorical distribution. 
 
 Let $\pi^{\theta},\pi^{(2)}$ be the categorical distribution of $\left\{\pi^\theta_i\right\}_{i\in \mathbb{N}}$ and $\left\{\pi^{(2)}_i\right\}_{i\in \mathbb{N}}$ respectively. Assume that $\pi^{\theta}$ is locally Lipschitz and the expectation of Lipschitz constant $Lipz(\theta,z)$ over $\theta$ and $z$ is less than infinity, $E_{z \sim p}[Lipz(\theta,z)]<+\infty$, with dual norm $\parallel \phi \parallel=\sup_{\parallel z\parallel \leq1} \mid\phi (z)\mid$, the following statements are proven in \cite{Arjovsky2017Wasserstein}:
\begin{enumerate}
    
    \item By the bounded convergence theorem, 
    
    $\mid W_1(\pi^{\theta},\pi^{(2)}) -W_1(\pi^{\theta^\prime},\pi^{(2)})\mid \leq W_1(\pi^{\theta},\pi^{\theta^\prime})\leq Lipz(\theta) \parallel\theta-\theta^\prime\parallel\rightarrow 0, \text{ as } \theta\rightarrow\theta^\prime$
    \item  $W(P_{\pi^\theta},P_{\pi^{(2)}})$ is continuous everywhere and differentiable almost everywhere.
    \item Statements 1-2 are false for the Jensen-Shannon divergence and all the KLs.
    \item The following statements are equivalent
    \begin{itemize}
        \item $\delta(\pi^{\theta}, \pi^{(2)}) \rightarrow 0$ with $\delta$ the total variation distance.
        \item $JS(\pi^{\theta}, \pi^{(2)}) \rightarrow 0$ with $JS$ the Jensen-Shannon divergence.
    \end{itemize}
    \item The following statements are equivalent
    \begin{itemize}
        \item $W(\pi^{\theta}, \pi^{(2)})\rightarrow 0 $.
        \item $\pi^{\theta}\rightarrow \pi^{(2)}$ where $\rightarrow$ represents convergence in variables.
    \end{itemize}
    \item $KL(\pi^{\theta}||\pi^{(2)}) \rightarrow 0$ or $KL(\pi^{(2)}||\pi^{\theta}) \rightarrow 0$ imply the statements in $4$.
    \item The statements in $5$ imply the statements in $4$.
\end{enumerate}

The statements 1-3 illustrate how the Wasserstein distance is more sensible than the JS divergence, TV distance, and KL divergences with respect to differentiability and continuity. In accordance with the definition of $NGMG$, it is established that $NGMG$ is continuous everywhere and differentiable almost everywhere. Because NGMG Eq.\eqref{ngmgfinal} is a vector of linear transformation of $(\pi^{\theta}-\pi^{(2)})^-$ and all elements in the vector are strictly positive, so that $\parallel ML(\pi^{\theta},\pi^{(2)}) \parallel_1- \parallel ML(\pi^{\theta^\prime},\pi^{(2)}) \parallel_1=\parallel ML(\pi^{\theta},\pi^{\theta^\prime})\parallel_1$. It proves that:

\begin{eqnarray}\notag
    \begin{aligned}
        \mid\parallel \text{NGMG}(L(\pi^{\theta},\pi^{(2)})) \parallel_1- \parallel \text{NGMG}(L(\pi^{\theta^\prime},\pi^{(2)})) \parallel_1\mid &=\parallel \text{NGMG}(L(\pi^{\theta},\pi^{\theta^\prime})) \parallel_1\\
        &\leq Lipz(\theta,z)\parallel \theta-\theta^\prime\parallel,
    \end{aligned}
\end{eqnarray}

\begin{proposition}
The NGMG is a linear function related to Wasserstein distance, which has the following representations:
\begin{equation}\label{WandNGMG}
W_1(\pi^{\theta},\pi^{(2)})=\mid \parallel FM^{-1} \text{NGMG}(L^+(\pi^{\theta},\pi^{(2)}))\parallel_1 \\
-\parallel FM^{-1} \text{NGMG}(L^-(\pi^{\theta},\pi^{(2)}))\parallel_1 \mid
\end{equation}
\end{proposition}
\begin{proof}
 Based on Eq.\eqref{gmmwasdist2}, we have:
\begin{equation*}
\overrightarrow{W_1}(\pi^{\theta},\pi^{(2)})^+=FB\cdot(\pi^{\theta}-\pi^{(2)}), ~~\overrightarrow{W_1}(\pi^{\theta},\pi^{(2)})^-=F(1-B)\cdot(\pi^{\theta}-\pi^{(2)})    
\end{equation*}
Let $L^+(\pi^{\theta},\pi^{(2)})=-B\cdot(\pi^{\theta}-\pi^{(2)})$ and $ L^-(\pi^{\theta},\pi^{(2)})=(1-B)\cdot(\pi^{\theta}-\pi^{(2)})$.
Then we have
\begin{equation*}
\overrightarrow{W_1}(\pi^{\theta},\pi^{(2)})^+=FM^{-1}\text{NGMG}(L^+(\pi^{\theta},\pi^{(2)}))),
\end{equation*}
\begin{equation*}
    \overrightarrow{W_1}(\pi^{\theta},\pi^{(2)})^-=FM^{-1}\text{NGMG}(L^-(\pi^{\theta},\pi^{(2)}))).
\end{equation*}
Together with Eq.\eqref{gmmwasdist2}, we proved Eq.\eqref{WandNGMG}.   
\end{proof}

As demonstrated in statements 4-7, convergence in the JS divergence, TV distance, and KL divergence implies convergence in the Wasserstein distance. NGMG exhibits a strong connection with the Wasserstein distance, suggesting that convergence in one implies convergence in the other. Therefore, the convergence properties of the Wasserstein distance can be extended to NGMG, and the reciprocal is also true.
\begin{proposition}
The following statements can be proved equivalent:
\begin{description}
    \item[(a)] $W(\pi^{\theta}, \pi^{(2)})\rightarrow 0 $.
    \item[(b)] $\parallel \text{NGMG}(\pi^{\theta}, \pi^{(2)}) \parallel_1 \rightarrow 0 $.
\end{description}
\end{proposition}
\begin{proof} 
From \textbf{(b)} to \textbf{(a)}:
If $\parallel \text{NGMG}(L^-(\pi^{\theta},\pi^{(2)}))\parallel_1 \rightarrow 0 $,

\begin{eqnarray*}
        \sum_j^n \mid \text{ngmg}(\pi^{\theta}_j,\pi^{(2)}_j)\mid=0. 
\end{eqnarray*}
And because $\mid \text{ngmg}(\pi^{\theta}_j,\pi^{(2)}_j) \mid \geq 0$, for all $j$,
$$\pi^{\theta}_j-\pi^{(2)}_j=0,\quad W(\pi^{\theta},\pi^{(2)})=0.$$
Vice versa for
$$\parallel \text{NGMG}(L^+(\pi^{\theta},\pi^{(2)}))\parallel_1 \rightarrow0.$$

From \textbf{(a)} to \textbf{(b)}:

If $W(\pi^{\theta},\pi^{(2)})\rightarrow0$, from Jensen Inequality:
\begin{eqnarray*}
    \begin{aligned}
    W(\pi^{\theta},\pi^{(2)})&=\int_{\mathbf{M}}\mid \sum_{n=1}^N (\pi_n^{\theta}-\pi_n^{(2)}) F_n(x)\mid dx \\
    &\geq \mid\sum_{n=1}^N (\pi_n^{\theta}-\pi_n^{(2)})\int_{\mathbf{M}}  F_n(x) dx\mid=0.
    \end{aligned}
\end{eqnarray*}
Based on GMM expansion setting
       $$ \int_{\mathbf{M}}  F_n(x) dx =A_n, A_1>A_2>A_3....>0.$$
We have 
$$ \pi_n^{\theta}-\pi_n^{(2)}=0, \quad n \in [1,N],$$
and
$$\text{NGMG}(L^-(\pi_n^{\theta},\pi_n^{(2)}))=\text{NGMG}(L^+(\pi_n^{\theta},\pi_n^{(2)}))=0.$$
\end{proof}

From the analysis presented, we can draw the following conclusion: Under the assumptions outlined in statement 2, we establish that $NGMG$ possesses the same desirable properties as the Wasserstein distance. Specifically, $NGMG$ is continuous and differentiable almost everywhere. Assuming convergence criteria are met for the Kullback-Leibler divergence, Total Variation distance, Jensen-Shannon distance, and Wasserstein distance, $NGMG$ is also shown to converge. Both $NGMG$ and the Wasserstein distance prove to be sensible cost functions, particularly for achieving convergence on low-dimensional manifolds where KL divergence, TV distance, and JS distance may fail to perform adequately.
\subsection{Numerical Experiments}

We ran various experiments to test our learning method. The first experiment involved learning to transform distribution A into distribution B. The second was a direct comparison of our method with binary cross-entropy in neural network training.

\subsubsection{Learning Density}
Moving between two distributions involves parametrization. If \( A \) is a Gaussian distribution but \( B \) is not, it is impossible to perfectly transform \( A \) into \( B \) due to the limited geometry of \( A \). In our approach, the density is approximated by a mixture of Gaussians with fixed means and variances. As shown in Section~2.3, the Wasserstein distance between any two distributions under Gaussian mixture expansion is given by Eq.~\eqref{gmmwasdist}. In this scenario, the parametrization of our model relates solely to the mixture coefficients \( \boldsymbol{\pi} \), and NGMG is applied to adjust these \( \boldsymbol{\pi} \)s.

\begin{algorithm}[h]
	\caption{Simple Learning Algorithm}
	\label{alg:2}
	\BlankLine
        \textbf{Requires: }$\sigma$, $\left\{\mu_i\right\}_{i=1}^N$, $\epsilon$,  $\pi$ is calculated using Gaussian mixture expansion method in Section 2.2. 

        \textbf{Initialise: }$\widehat{\pi}$
        
	\textbf{While }{$L>\epsilon $ } \textbf{do}
 
          \textbf{            }1.Calculate $\text{NGMG}(L(\pi,\widehat{\pi}),\sigma)$

          \textbf{            }2.Update $\widehat{\pi}=\widehat{\pi}+\text{NGMG}(L(\pi,\widehat{\pi}),\widehat{\pi},\sigma)$

          \textbf{            }3.Normalise$\widehat{\pi}$.
       
       \textbf{return}  $\widehat{\pi}$
\end{algorithm}

\begin{figure}[ht]
	\centering
		\includegraphics[scale=0.45]{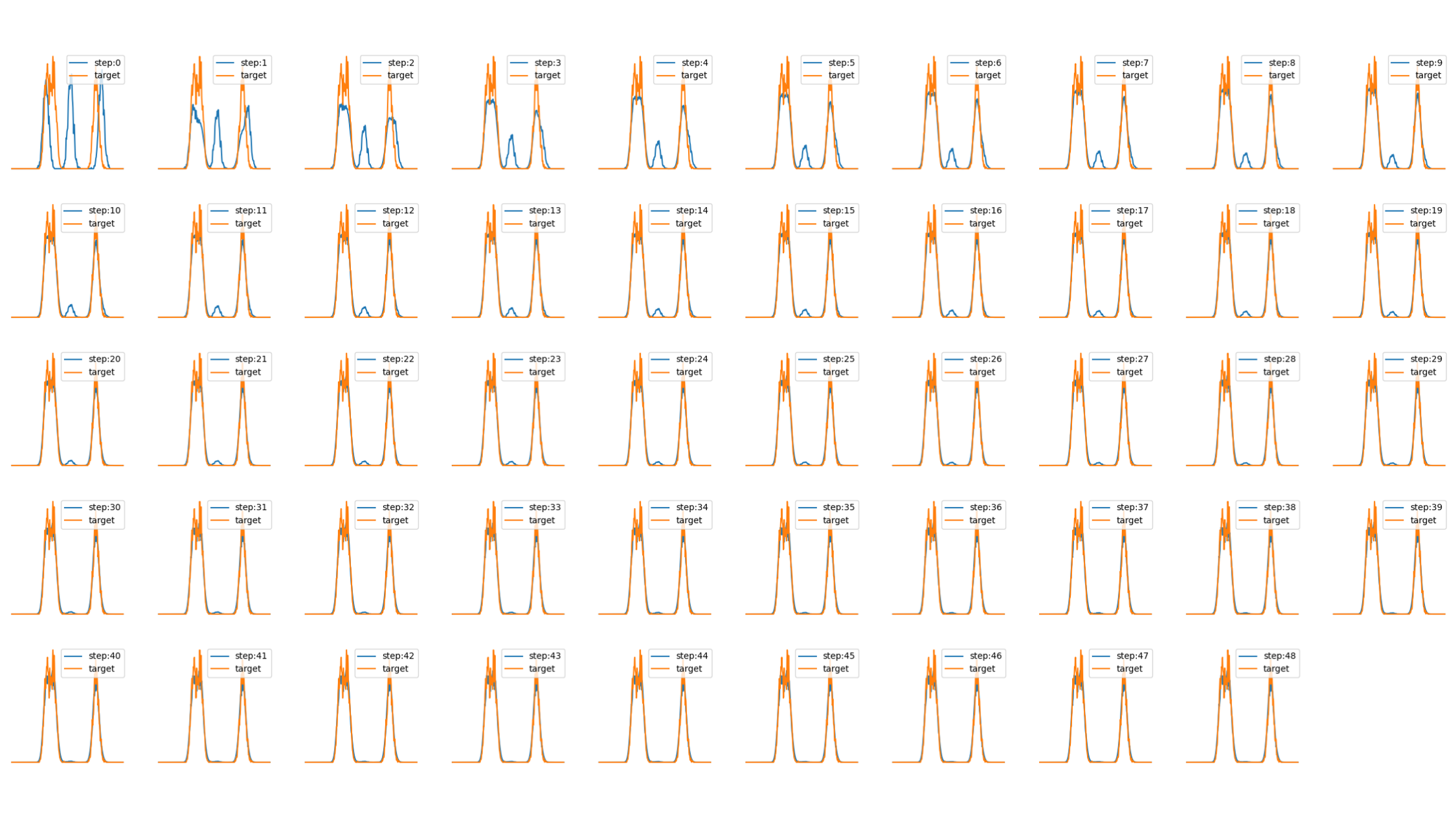}
	\caption{Learning two distribution by NGMG. This shows each learning iteration in algorithm.\ref{alg:2}. Orange line is the target $\pi$ and blue line is the initial $\widehat{\pi}$. }
	\label{learnsteps}
\end{figure}

Figure~\ref{learnsteps} shows the learning process using NGMG. The orange curve represents the target \( \boldsymbol{\pi} \), and the blue curve represents \( \widehat{\boldsymbol{\pi}} \). After 10 steps, \( \widehat{\boldsymbol{\pi}} \) is already relatively close to the target \( \boldsymbol{\pi} \). Because the loss function \( L \) decreases, the learning process in later steps becomes less aggressive. The parameter \( \sigma \) is included in the equation because it can be parameterized during training to provide more control. Figure~\ref{learnsteps} showcases an example with a constant \( \sigma \). We can also develop a learning scheme that updates \( \sigma \) based on certain criteria. The parameter \( \sigma \) controls the modified covariance matrix \( M \). Compared to other cost functions, NGMG has the distinct feature of allowing more control over the learning process.

\subsubsection{Learning in Neural Network}

NGMG is a function that provides a gradient for learning. In the aforementioned density learning experiment, normalization is applied to ensure the probability condition, such that the sum of \( \boldsymbol{\pi} \)s equals one. However, in a neural network, normalization at each training step is not feasible. Instead of training with binary cross-entropy, we introduce negative Gaussian mixture gradient (NGMG) entropy. Intuitively, just as the Earth Mover's (Wasserstein) distance moves the earth, NGMG here moves the entropy.

Shannon entropy:
\begin{equation}\label{entropyShanon}
    H(x)=-\sum_{x\in\mathcal{X}} p(x) \log(p_\theta(x)).
\end{equation}


Binary cross-entropy:

\begin{equation}\label{entropyBC}
    H(x)=- (p(x) \log(p_\theta(x))+(1-p(x))\log(1-p_\theta(x))).
\end{equation}


\textbf{Negative Gaussian mixture gradient entropy}:

\begin{equation}\label{entropyNGMG}
   H(x)=-\parallel\text{NGMG}(L(p(x),p_{\theta}(x)),\sigma)\cdot \log(p_{\theta}(x))\parallel_1. 
\end{equation}

\begin{figure}[ht]
	\centering
		\includegraphics[scale=0.45]{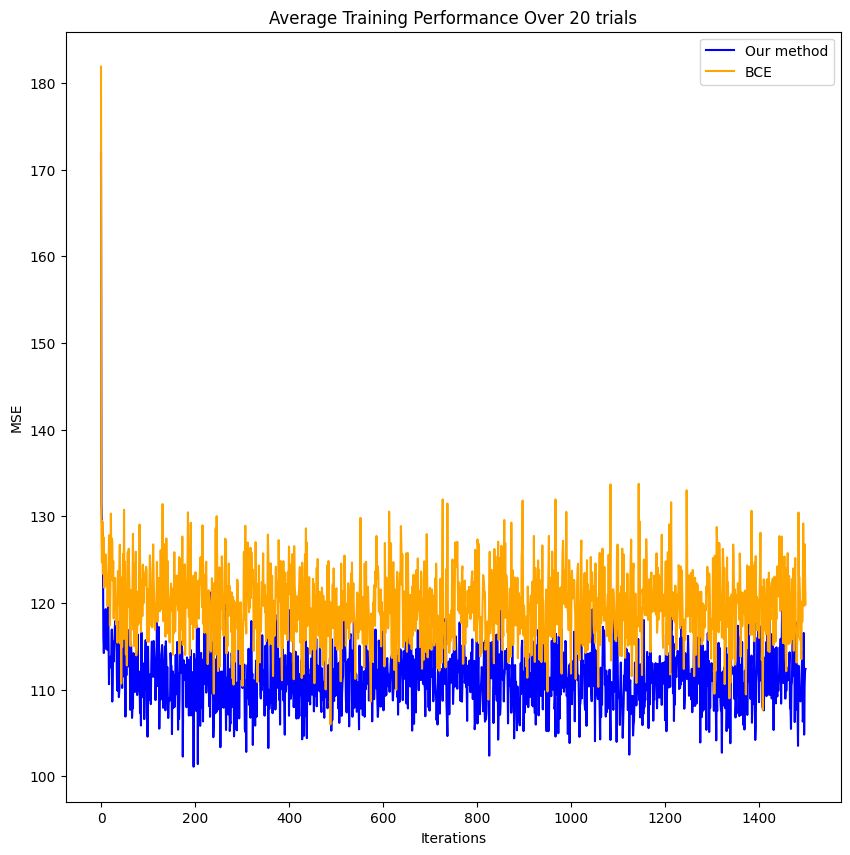}
	\caption{NGMG Entropy vs. BCE. Orange line represent the average training loss of model trained in BCE. Blue line is our result. MSE metric is applied to compare the learning result directly. Results are averaged over 20 trials. Every trial trained with 1500 iterations. }
	\label{NGNGVSBCE}
\end{figure}

In this experiment, a simple neural network with a single input is tested to fit the feature attributes of the CelebA dataset. Figure~\ref{NGNGVSBCE} shows the sum of squared errors between the predicted probabilities and feature attributes in a batch on the test dataset. The blue line represents the performance of NGMG entropy, and the orange line represents the performance of binary cross-entropy. NGMG entropy clearly shows better performance than binary cross-entropy. Equations~\eqref{entropyShanon}, \eqref{entropyNGMG}, and \eqref{entropyBC} indicate that our method combines Wasserstein distance and Shannon entropy, with NGMG used to control the entropy term \( \log(p_{\theta}(x)) \).

\section{Conclusions}

We have presented a distributional conditioning mechanism wherein latent variables are treated as random variables. A Gaussian mixture model is used to construct the latent distribution. The data (real) distribution is approximated by the latent distribution, and a diffusion model is trained on the CelebA dataset to demonstrate this conditioning mechanism. Our generation results are promising and indicate potential for further development. Conditioning on features or classes can significantly affect model performance. We theoretically and experimentally show that conditioning the model on features produces fewer defective generations than conditioning on classes. Additionally, we present a diffusion model with a classifier and propose a novel distance function, Negative Gaussian Mixture Gradient (NGMG), to train this diffusion model. We prove that NGMG shares the same benefits as the Wasserstein distance. It is more sensitive than KL-divergence, Jensen-Shannon divergence, and Total Variation when learning distributions supported by low-dimensional manifolds. Our experiments demonstrate favorable results compared to binary cross-entropy.
\section*{Acknowledgement}
This work was partly supported by NSFC grant 12141107.
\newpage

\appendix 
\section{Sheets Of Samples}\label{Appendix1}

\begin{figure}[ht]
	\centering
		\includegraphics[scale=1.0]{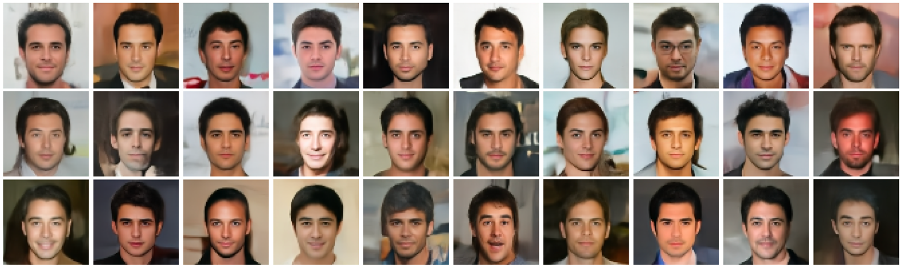}
	\caption{ Random samples condition on: 
Bushy Eyebrows, Male, Pointy Nose, Straight Hair, Young, Attractive...}
\end{figure}
\begin{figure}[ht]
	\centering
		\includegraphics[scale=1.0]{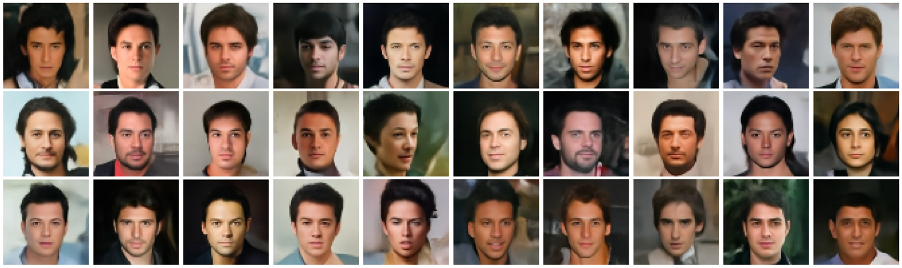}
	\caption{Random samples condition on: 
5 o'Clock Shadow, Arched Eyebrows, Bags Under Eyes, Male, Black Hair… }
\end{figure}
\begin{figure}[ht]
	\centering
		\includegraphics[scale=1.0]{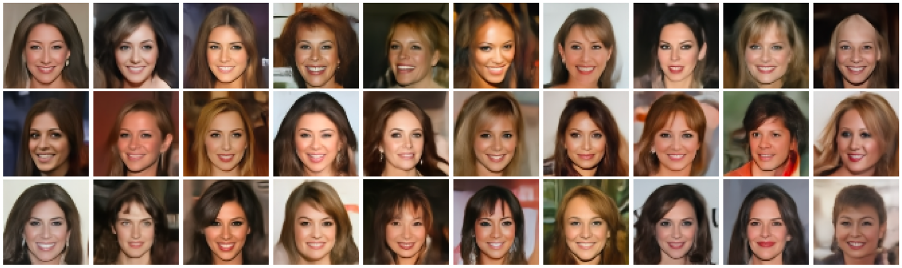}
	\caption{Random samples condition on: 
Arched Eyebrows, Attractive, Bangs, Big Lips, Heavy Makeup, High Cheekbones, …}
	\label{}
\end{figure}
\begin{figure}[ht]
	\centering
		\includegraphics[scale=1.0]{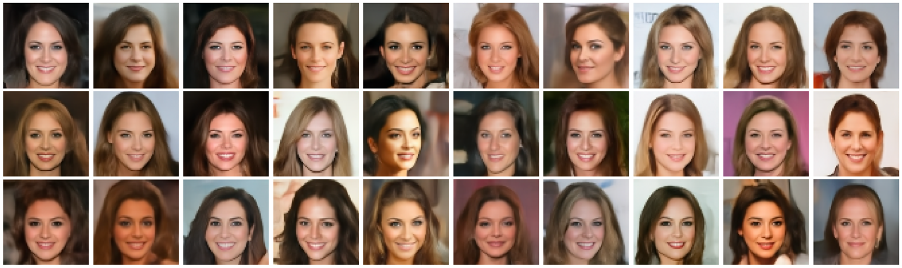}
	\caption{Random samples condition on: 
Arched Eyebrows, Attractive, Brown Hair, Heavy Makeup, Smiling, Young … }
\end{figure}

\newpage

\section{Diffusion Model Sampling}\label{Appendix2}

In our experiments, a simple strategy for forward diffusion process as well as denoising process is applied. Our notation follows \cite{Ho2020Denoising}.

\begin{eqnarray*}
    \alpha_t=1-\beta_t,\\
    \overline{\alpha}_t:= \prod_{s=1}^{t}\alpha_s,\\
\end{eqnarray*}
For $\epsilon \sim \phi(0,I)$, 
\begin{equation*}
        x_t=\sqrt{\overline{\alpha}_t}x_0+\sqrt{(1-\overline{\alpha}_t)}\epsilon.
\end{equation*}

Sampling process in \cite{Ho2020Denoising} is given by following equation:
\begin{equation*}
    x_{t-1}=\frac{1}{\sqrt{\alpha_t}}\left(x_t- \frac{1-\alpha_t}{\sqrt{1-\overline{\alpha}_t}}\epsilon_\theta\left(x_t,t\right)\right)+\sigma_t z,
\end{equation*}
where $z\sim \phi(0,I)$.
Denoising process is taking the same path that we use in training Eq.\eqref{denoise} which we calculate $x_0$ at every steps and noise up to $x_{t-1}$ autoregressively. 
\begin{eqnarray}\label{denoise}
    \begin{aligned}
        x_0^* &=\frac{x_t-\sqrt{(1-\overline{\alpha})}\epsilon_\theta\left(x_t,t\right)}{\sqrt{\overline{\alpha}}},\\
    x_{t-1}&=\sqrt{\overline{\alpha}}x_0^*+\sqrt{(1-\overline{\alpha})}\epsilon
    \end{aligned}
\end{eqnarray}

Figure .\ref{picx0Sampling} shows the generating result at each steps. Total steps of forward diffusion process in our experiment for CelebA is 100. Backward autoregressive 
denosing process is start at a step 100 with complete Gaussian noise. It shows that our model quickly find our target images and making detail refinement at each sampling steps. The predicted $x_0^*$ from steps 80-100 reach similar generation quality. We apply the same simple loss function for model training\citep{Ho2020Denoising, Rombach2022High}.

\begin{equation*}
    L_{\text{simple}}(\theta):=E_{t,x_0,\epsilon}\left[\parallel \epsilon-\epsilon_\theta(\sqrt{\overline{\alpha}}x_0+\sqrt{(1-\overline{\alpha})}\epsilon,t)\parallel\right]
\end{equation*}




\begin{thebibliography}{99}


 \bibitem[Sohl-Dickstein et al.(2015)]{Sohl2015Deep}
Sohl-Dickstein J., Weiss E., Maheswaranathan N., Ganguli S.,
Deep unsupervised learning using nonequilibrium thermodynamics.
\emph{International conference on machine learning},  
pp.2256-2265 PMLR. 2015.

\bibitem[Ho et al.(2020)]{Ho2020Denoising}
Ho J., Jain A., Abbeel P.,
Denoising diffusion probabilistic models.
\emph{NeurIPS}, 2020.

\bibitem[Chen et al.(2021)]{Chen2021Wavegard}
Chen N., Zhang Y., Zen H., Weiss R., Norouzi M., Chan W.,
Wavegrad: Estimating gradients for waveform generation.
\emph{ICLR}, 2021,
OpenReview.net.

\bibitem[Kingma et al.(2021)]{Kingma2021Variational}
Kingma D., Salimans T., Poole B., Ho J.,
Variational diffusion models.
\emph{CoRR, abs/2107.00630}, 2021.

\bibitem[Dhariwal and Nichol(2021)]{Dhariwal2021Diffusion}
Dhariwal P., Nichol A.,
Diffusion models beat gans on image synthesis,
\emph{Advances in neural information processing systems,34 }, 
pp.8780-8794, 2021.

\bibitem[Ho et al.(2022)]{Ho2022Cascaded}
Ho, J., Saharia, C., Chan, W., Fleet, D.J., Norouzi, M. and Salimans, T.,
Cascaded diffusion models for high fidelity image generation.
\emph{The Journal of Machine Learning Research, 23(1)}, 
pp.2249-2281, 2022.

\bibitem[Saharia et al.(2022)]{Saharia2022Image}
Saharia, C., Ho, J., Chan, W., Salimans, T., Fleet, D.J. and Norouzi, M.,
Image super-resolution via iterative refinement.
\emph{IEEE Transactions on Pattern Analysis and Machine Intelligence, 45(4)}, 
pp.4713-4726, 2022.

\bibitem[Song and Ermon(2019)]{Song2019Generative}
Song Y., Ermon S.,
Generative modeling by estimating gradients of the data distribution.
\emph{Advances in neural information processing systems, 32}, 
2019.

\bibitem[Nichol and Dhariwal(2021)]{Nichol2021Improved}
Nichol AQ, Dhariwal P.,
Improved denoising diffusion probabilistic models.
\emph{International Conference on Machine Learning}, PMLR, 
pp.8162-8171, 2021.

\bibitem[Song and Ermon(2020)]{Song2020Improved}
Song Y, Ermon S.,
Improved techniques for training score-based generative models.
\emph{Advances in neural information processing systems, 33},  
pp.12438-12448, 2020.

\bibitem[Song et al.(2020)]{Song2020Score}
Song, Y., Sohl-Dickstein, J., Kingma, D.P., Kumar, A., Ermon, S. and Poole, B.,
Score-based generative modeling through stochastic differential equations.
\emph{arXiv preprint arXiv:2011.13456.},  
2020.

\bibitem[Rombach et al.(2022)]{Rombach2022High}
Rombach R, Blattmann A, Lorenz D, Esser P, Ommer B.,
High-resolution image synthesis with latent diffusion models.
\emph{In Proceedings of the IEEE/CVF conference on computer vision and pattern recognition},  
pp. 10684-10695. 2022.

\bibitem[Kong et al.(2020)]{Kong2020Diffwave}
Kong, Z., Ping, W., Huang, J., Zhao, K. and Catanzaro, B.,
Diffwave: A versatile diffusion model for audio synthesis.
\emph{arXiv preprint arXiv:2009.09761}, 2020.

\bibitem[Mittal et al.(2021)]{Mittal2021Symbolic}
Mittal, G., Engel, J., Hawthorne,C., and Simon, I.,
Symbolic music generation with diffusion models.
\emph{arXiv preprint arXiv:2103.16091}, 2021.


\bibitem[Ronneberger et al.(2015)]{Ronneberger2015Unet}
    O. Ronneberger, P. Fischer, and T. Brox,
    "U-net: Convolutional networks for biomedical image segmentation."
    \emph{In Medical Image Computing and Computer-Assisted Intervention–MICCAI 2015: 18th International Conference}, 
    pp. 234-241, Springer International Publishing, 2015.

\bibitem[Kingma and Welling(2014)]{kingma2014auto}
    D. Kingma and M. Welling,
    "Auto-encoding variational bayes."
    \emph{In 2nd International Conference on Learning Representations, ICLR}, 2014.

\bibitem[Am{\'e}ndola et al.(2015)]{amendola2015maximum}
    C. Am{\'e}ndola, M. Drton, and B. Sturmfels,
    "Maximum likelihood estimates for Gaussian mixtures are transcendental."
    \emph{MACIS}, 2015, pp. 579--590.

\bibitem[Abbi et al.(2008)]{abb2008}
    R. Abbi, E. El-Darzi, C. Vasilakis, and P. Millard,
    "Analysis of stopping criteria for the EM algorithm in the context of patient grouping according to length of stay."
    \emph{IEEE Intelligent Systems}, vol. 1, 2008, pp. 3-9.

\bibitem[Biernackia et al.(2003)]{bier2003}
    C. Biernackia, G. Celeuxb, and G. Govaertc,
    "Choosing starting values for the EM algorithm for getting the highest likelihood in multivariate Gaussian mixture models."
    \emph{Computational Statistics and Data Analysis}, vol. 41, 2003, pp. 561-575.

\bibitem[Blömer and Bujna(2013)]{blo2013}
    J. Blömer and K. Bujna,
    "Simple methods for initializing the EM algorithm for Gaussian mixture models."
    \emph{CoRR}, 2013.

\bibitem[Chi et al.(2016)]{jin2016local}
    J. Chi, Y. Zhang, S. Balakrishnan, M. Wainwright, and M. Jordan,
    "Local maxima in the likelihood of Gaussian mixture models: Structural results and algorithmic consequences."
    \emph{NIPS}, vol. 29, 2016.

\bibitem[Kontaxakis and Tzanakos(1992)]{kon1992}
    G. Kontaxakis and G. Tzanakos,
    "Study of the convergence properties of the EM algorithm-a new stopping rule."
    \emph{IEEE NSS/MIC}, 1992, pp. 1163-1165.

\bibitem[Kontaxakis and Tzanakos(1993)]{kon1993}
    G. Kontaxakis and G. Tzanakos,
    "Further study of a stopping rule for the EM algorithm."
    \emph{NEBEC}, 1993, pp. 52-53.

\bibitem[Kwedlo(2013)]{kwe2013}
    W. Kwedlo,
    "A new method for random initialization of the EM algorithm for multivariate Gaussian mixture learning."
    \emph{CORES}, 2013, pp. 81-90.

\bibitem[McKenzie and Alder(1994)]{mck1994}
    P. McKenzie and M. Alder,
    "Initializing the EM algorithm for use in Gaussian mixture modelling."
    \emph{Pattern Recognition}, 1994, pp. 91-105.

\bibitem[Paclík and Novovičová(2001)]{pac2001}
    P. Paclík and J. Novovičová,
    "A new method for random initialization of the EM algorithm for multivariate Gaussian mixture learning."
    \emph{ANNs/GAs}, 2001, pp. 406-409.

\bibitem[Shireman et al.(2017)]{shi2017}
    E. Shireman, D. Steinley, and M. Brusco,
    "Examining the effect of initialization strategies on the performance of Gaussian mixture modeling."
    \emph{Behavior Research Methods}, vol. 49,1:282-293, 2017.

\bibitem[Srebro(2007)]{srebro2007there}
    N. Srebro,
    "Are there local maxima in the infinite-sample likelihood of Gaussian mixture estimation?"
    \emph{COLT}, 2007, pp. 628--629.


\bibitem[Lu et al.(2023)]{Lu2023Efficient}
    Lu W., Ding D., Wu F.,Yuan G.,
    An efficient Gaussian mixture model and its application to neural network.
    \emph{Preprint:202302.0275.v2}, 2023.

\bibitem[Lu and Wu(2023)]{Lu2023Efficient2}
    Lu W., Wu X., Ding D., Yuan G.,
    An Efficient 1 Iteration Learning Algorithm for Gaussian Mixture Model And Gaussian Mixture Embedding For Neural Network.
    \emph{arXiv preprint arXiv:2308.09444}, 2023.

\bibitem[Devlin et al.(2019)]{Devlin2019Bert}
    Devlin J., Chang M., Lee K., Toutanova K.,
    BERT: pre-training of deep bidirectional transformers for language understanding.
    \emph{In Proceedings of NAACL-HLT (Vol. 1, p. 2).}, 2019.

\bibitem[Liu et al.(2015)]{Liu2015Face}
    Liu Z., Luo P., Wang X., Tang X.,
    Deep Learning Face Attributes in the Wild.
    \emph{Proceedings of International Conference on Computer Vision (ICCV)}, 2015.

\bibitem[Krizhevsky and Hinton(2009)]{Krizhevsky2009Learning}
    Krizhevsky A., Hinton G.,
    Learning multiple layers of features from tiny images.
    2009.

\bibitem[Dumoulin et al.(2017)]{Dumoulin2017A}
    Dumoulin V., Shlens J., Kudlur M.,
    A learned representation for artistic style.
    \emph{arXiv preprint arXiv:1610.07629}, 2017.

\bibitem[De Vries et al.(2017)]{Vries2017Modulating}
    De Vries H., Strub F., Mary J., Larochelle H., Pietquin O., Courville A.C.,
    Modulating early visual processing by language.
    \emph{Advances in Neural Information Processing Systems, 30}, 2017.

\bibitem[Miyato and Koyama(2018)]{Miyato2017Cgan}
    Miyato T., Koyama M.,
    cGANs with projection discriminator.
    \emph{arXiv preprint arXiv:1802.05637}, 2018.

\bibitem[Lucic et al.(2019)]{Lucic2019High}
    Lucic M., Tschannen M., Ritter M., Zhai X., Bachem O., Gelly S.,
    High-fidelity image generation with fewer labels.
    \emph{In International conference on machine learning (pp. 4183-4192). PMLR}, 2019.

\bibitem[Dash et al.(2017)]{Dash2017Tac}
    Dash A., Gamboa J.C.B., Ahmed S., Liwicki M., Afzal M.Z.,
    Tac-gan-text conditioned auxiliary classifier generative adversarial network.
    \emph{arXiv preprint arXiv:1703.06412}, 2017.

\bibitem[Lang et al.(2021)]{Lang2021Explaining}
    Lang O., Gandelsman Y., Yarom M., Wald Y., Elidan G., Hassidim A., Freeman W.T., Isola P., Globerson A., Irani M., Mosseri I.,
    Explaining in style: Training a gan to explain a classifier in stylespace.
    \emph{In Proceedings of the IEEE/CVF International Conference on Computer Vision (pp. 693-702)}, 2021.

\bibitem[Goodfellow et al.(2020)]{Goodfellow2020Generative}
    Goodfellow I., Pouget-Abadie J., Mirza M., Xu B., Warde-Farley D., Ozair S., Courville A., Bengio Y.,
    Generative adversarial networks,
    \emph{Communications of the ACM, 63(11):139-144}, 2020.

\bibitem[Wu et al.(2019)]{Wu2019Logan}
    Wu Y., Donahue J., Balduzzi D., Simonyan K., Lillicrap T.,
    Logan: Latent optimisation for generative adversarial networks.
    \emph{arXiv preprint arXiv:1912.00953}, 2019.

\bibitem[Karras et al.(2020)]{Karras2020Analyzing}
    Karras T., Laine S., Aittala M., Hellsten J., Lehtinen J., Aila T.,
    Analyzing and improving the image quality of stylegan.
    \emph{In Proceedings of the IEEE/CVF conference on computer vision and pattern recognition (pp. 8110-8119)}, 2020.

\bibitem[Brock et al.(2018)]{Brock2018Large}
    Brock A., Donahue J., Simonyan K.,
    Large scale GAN training for high fidelity natural image synthesis.
    \emph{arXiv preprint arXiv:1809.11096}, 2018.

\bibitem[Arjovsky et al.(2017)]{Arjovsky2017Wasserstein}
    Arjovsky M., Chintala S., Bottou L.,
    Wasserstein generative adversarial networks.
    \emph{In International conference on machine learning (pp. 214-223). PMLR}, 2017.

\bibitem[Arjovsky and Bottou(2017)]{Arjovsky2017Towards}
    Arjovsky M., Bottou L.,
    Towards principled methods for training generative adversarial networks.
    \emph{arXiv preprint arXiv:1701.04862}, 2017.

\bibitem[Razavi et al.(2019)]{Razavi2019Generating}
    Razavi A., Van den Oord A., Vinyals O.,
    Generating diverse high-fidelity images with vq-vae-2,
    \emph{In Advances in Neural Information Processing Systems, 32}, 2019.

\bibitem[Van Den Oord and Vinyals(2017)]{VanDenOord2017Neural}
    Van Den Oord A., Vinyals O.,
    Neural discrete representation learning.
    \emph{In Advances in Neural Information Processing Systems, 30}, 2017.

\bibitem[Mirza and Osindero(2014)]{Mirza2014Conditional}
    Mirza M., Osindero S.,
    Conditional generative adversarial nets.
    \emph{arXiv preprint arXiv:1411.1784}, 2014.

\bibitem[Villani(2009)]{Villani}
    Villani C.,
    Optimal transport: old and new.
    \emph{Berlin: Springer}, 2009.

\bibitem[Nachmani et al.(2021)]{Nachmani2021Non}
    Nachmani E., Roman R.S., Wolf L.,
    Non gaussian denoising diffusion models.
    \emph{arXiv preprint arXiv:2106.07582}, 2021.

\bibitem[Kolouri et al.(2018)]{Kolouri2018Sliced}
    Kolouri S., Rohde G.K., Hoffmann H.,
    Sliced wasserstein distance for learning gaussian mixture models.
    \emph{In Proceedings of the IEEE Conference on Computer Vision and Pattern Recognition (pp. 3427-3436)}, 2018.

\end{thebibliography}


\end{document}